\documentclass{article} 
\usepackage{iclr2019_conference,times}
\usepackage{amsmath,amssymb,mathtools,amsthm}
\usepackage{algpseudocode}
\usepackage{breqn}
\usepackage{algorithm}
\usepackage{subcaption}
\usepackage{cleveref}
\usepackage{mdframed}
\usepackage{bbm}
\usepackage{wrapfig}
\usepackage{scrextend}
\usepackage[font={small}]{caption}
\usepackage{enumitem}
\usepackage{url}
\usepackage{hyperref}

\hypersetup{
  colorlinks   = true, 
  urlcolor     = blue, 
  linkcolor    = blue, 
  citecolor   = red 
}
\usepackage{graphicx}

\newcommand{\dxbydy}[2]{\frac{\partial#1}{\partial #2}}

\newcommand{\DxbyDy}[2]{\frac{d #1}{d #2}}
\newcommand{\DbyDt}{\frac{d}{d \theta}}
\newcommand{\DbyDp}{\frac{d}{d \phi}}

\newcommand{\sg}[1]{\widetilde{#1}}
\newcommand{\hV}{{\hat{V}^\phi}}

\newcommand{\cInp}{\Theta}
\newcommand{\cX}{\mathcal{X}}

\newcommand{\cS}{\mathcal{S}}

\newcommand{\cB}{\mathcal{B}}
\newcommand{\cC}{\mathcal{C}}

\newcommand{\cD}{\mathcal{D}}
\newcommand{\cJ}{\mathcal{J}}

\newcommand{\cL}{\mathcal{L}}
\newcommand{\cH}{\mathcal{H}}

\newcommand{\cG}{\mathcal{G}}
\newcommand{\cV}{\mathcal{V}}

\newcommand{\cW}{\mathcal{W}}

\newcommand{\bcG}{\overline{\mathcal{G}}}

\newcommand{\inp}{\theta}

\newcommand{\comment}[1]{}

\newcommand{\dgiven}[1][]{#1\mkern-0.8mu\big\vert}
\newcommand{\bigdgiven}[1][]{#1\mkern-0.8mu\Big\vert}

\newcommand{\E}{\mathbb{E}}
\newcommand{\Var}{\mathrm{Var}}

\newcommand{\mb}{\mathbf}
\newcommand{\bx}{\mathbf{x}}

\newtheorem{theorem}{Theorem}
\newtheorem{definition}{Definition}

\newtheorem{lemma}{Lemma}

\newtheorem{property}{Property}
\newtheorem{corollary}{Corollary}
\newtheorem{example}{Example}

\title{Credit Assignment Techniques\\ in Stochastic Computation Graphs}

\author{Th\'eophane Weber\thanks{equal contribution.} , Nicolas Heess$^*$, Lars Buesing \& David Silver \\
DeepMind\\
\texttt{\{theophane, heess, lbuesing, davidsilver\}@google.com} 
}
\iclrfinalcopy
\begin{document}

\maketitle

\begin{abstract}

\emph{Stochastic computation graphs} (SCGs) provide a formalism to represent structured optimization problems arising in artificial intelligence, including supervised, unsupervised, and reinforcement learning. Previous work has shown that an unbiased estimator of the gradient of the expected loss of SCGs can be derived from a single principle.
However, this estimator often has high variance and requires a full model evaluation per data point, making this algorithm costly in large graphs.
In this work, we address these problems by generalizing concepts from the reinforcement learning literature. We introduce the concepts of value functions, baselines and critics for arbitrary SCGs, and show how to use them to derive lower-variance gradient estimates from partial model evaluations, paving the way towards general and efficient credit assignment for gradient-based optimization. In doing so, we demonstrate how our results unify recent advances in the probabilistic inference and reinforcement learning literature.
\end{abstract}

\section{Introduction} \label{intro}

The machine learning community has recently seen breakthroughs in challenging problems in classification, density modeling, and reinforcement learning (RL). 
To a large extent, successful methods have relied on gradient-based optimization (in particular on the backpropagation algorithm \citep{rumelhart1985learning}) for credit assignment, i.e.\ for answering the question how individual parameters (or units) affect the value of the objective. 
Recently, \citet{schulman2015gradient} have shown that such problems can be formalized as optimization in stochastic computation graphs (SCGs).
Furthermore, they derive a general gradient estimator that remains valid even in the presence of stochastic or non-differentiable computational nodes.
This unified view reveals that numerous previously proposed, domain-specific gradient estimators, such as the likelihood ratio estimator \citep{glasserman1992smoothing}, also known as `REINFORCE'~\citep{williams1992simple},  as well as the pathwise derivative estimator, also known as the ``reparameterization trick'' \citep{glasserman1991gradient, kingma2014efficient, rezende2014stochastic}, can be regarded as instantiations of the general SCG estimator.
While theoretically correct and conceptually satisfying, the resulting estimator often exhibits high variance in practice, and significant effort has gone into developing techniques to mitigate this problem \citep{ng1999policy,sutton2000policy,schulman2015high,arjona2018rudder}. Moreover, like backpropagation, the general SCG estimator requires a full forward and backward pass through the entire graph for each gradient evaluation, making the learning dynamics global instead of local. This can become prohibitive for models consisting of hundreds of layers, or recurrent model trained over long temporal horizons.

In this paper, starting from the SCG framework, we target those limitations, by introducing a collection of results which unify and generalize a growing body of results dealing with credit assignment. In combination, they lead to a spectrum of approaches that provide estimates of model parameter gradients for very general deterministic or stochastic models.
Taking advantage of the model structure, they allow to trade off bias and variance of these estimates in a flexible manner. Furthermore, they provide mechanisms to derive local and asynchronous gradient estimates that relieve the need for a full model evaluation. Our results are borrowing from and generalizing a class of methods popular primarily in the reinforcement learning literature, namely that of learned approximations to the surrogate loss or its gradients, also known as \emph{value functions}, \emph{baselines} and \emph{critics}. As new models with increasing structure and complexity are actively being developed by the machine learning community, we expect these methods to contribute powerful new training algorithms in a variety of fields such as hierarchical RL, multi-agent RL, and probabilistic programming.

This paper is structured as follows. 
We review the stochastic computation graph framework and recall the core result of \cite{schulman2015gradient} in section~\ref{sec:prelim}. In section \ref{sec:valueCritic} we discuss the notions of value functions, baselines, and critics in arbitrary computation graphs, and discuss how and under which conditions they can be used to obtain both lower variance and local estimates of the model gradients, as well as local learning rules for other value functions. In section \ref{sec:gradientCritic} we provide similar results for gradient critics, i.e.\ for estimates or approximations of the downstream loss gradient. In section~\ref{sec:Examples} we go through many examples of more structured SCGs arising from various applications and see how our techniques allow to derive different estimators. 
In section \ref{sec:Menu}, we discuss how the techniques and concepts introduced in the previous sections can be used and combined in different ways to obtain a wide spectrum of different gradient estimators with different strengths and weaknesses for different purposes. We conclude by investigating a simple chain graph example in detail in section~\ref{sec:madness}.

\textbf{Notation for derivatives}

We use a `physics style' notation by representing a function and its output by the same letter. For any two variables $x$ and $y$ in a computation graph, we use the partial derivative $\frac{\partial y}{\partial x}$ to denote the direct derivative of $y$ with respect to $x$, and the $\frac{dy}{dx}$ to denote the total derivative of $y$ with respect to $x$, taking into account all paths (or effects) from $x$ on $y$; we use this notation even if $y$ is still effectively a function of multiple variables. For any variable $x$, we let $\sg{x}$ denote the value of $x$ which we treat as a constant in the gradient formula; i.e.\ $\sg{x}$ can be thought of a the output of a `function' which behaves as the identity, but has gradient zero everywhere\footnote{such an operation is often called `stop gradient' in the deep learning community.}. Finally, we only use the derivative notation for deterministic computation, the gradient of any sampling operation is assumed to be $0$.

\emph{All proofs are omitted from the main text and can be found in the appendix.}

\section{Gradient estimation for expectation of a single function}
\label{sec:prelim}

An important class of problems in machine learning can be formalized as  optimizing an expected loss $\E_{x_1,\ldots,x_n \sim p(x_1,\dots,x_n|\theta)} [ \ell(x_1,x_2,\dots, \theta)]$ over parameters $\theta$, where \emph{both} the sampling distribution $p(.|\theta)$ as well as the loss function $\ell$ can depend on $\theta$.
As we explain in greater detail in the appendix, concrete examples of this setup are reinforcement learning (where $p$ is a composition of known policy and potentially unknown system dynamics), and variational autoencoders (where $p$ is a composition of data distribution and inference network); cf. Fig.~\ref{fig:example_models}. Because of the dependency of the distribution on $\theta$, backpropagation does not directly apply to this problem. 

Two well-known estimators, the score function estimator and the pathwise derivative, have been proposed in the literature. Both turn the gradient of an expectation into an expectation of gradients and hence allow for unbiased Monte Carlo estimates by simulation from known distributions, thus opening the door to stochastic approximations of gradient-based algorithms such as gradient descent \citep{robbins1985stochastic}. 

\begin{minipage}[h]{1.\linewidth}
    \centering
    \includegraphics[width=0.85\textwidth]{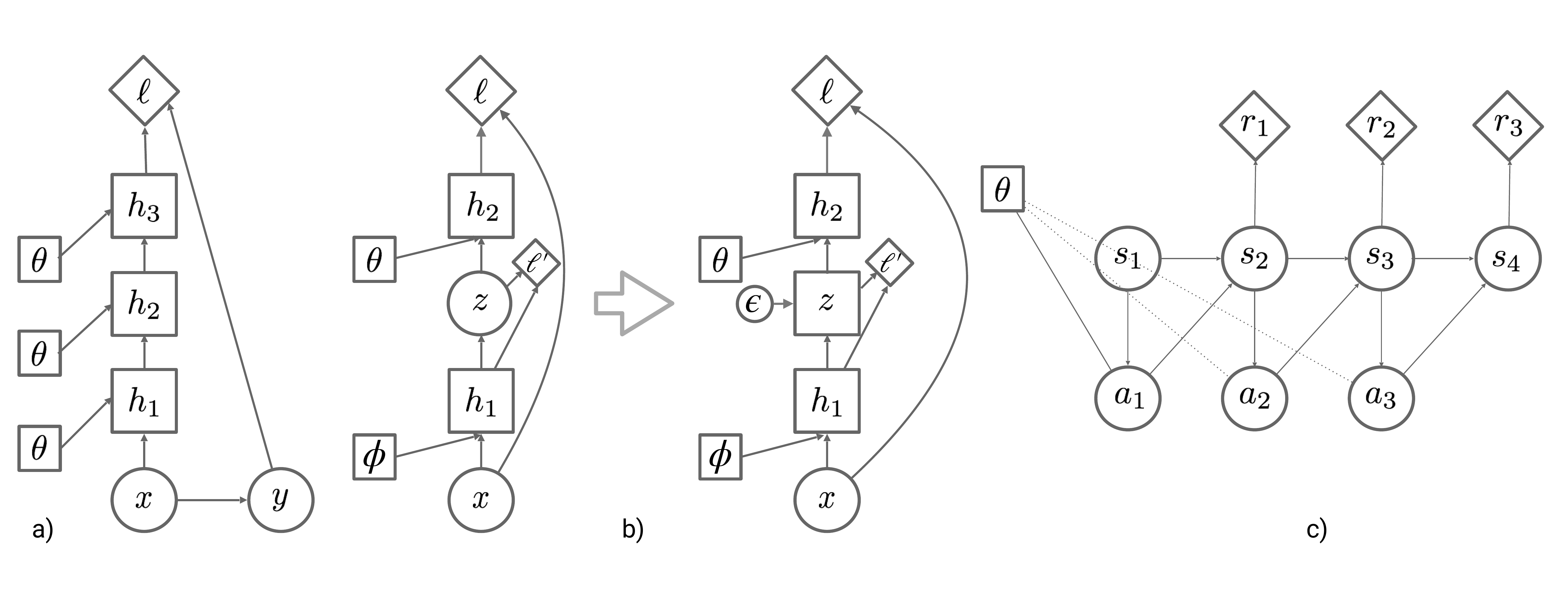}
    \captionof{figure}{\emph{Typical models as SCGs.}
    (a) {\bf Supervised learning.} The loss is $\E_{x,y}\left [ \log p(y | x, \theta) \right]$ where $x,y \sim p_D$ are drawn from the data distribution. $p(y|x,\theta)$ is realized as a multi-layer neural network with hidden layers $h_1, h_2, h_3$ and parameters $\theta$.
    (b) {\bf Variational autoencoder.} The observed data $x \sim p_D$ is fed into an inference network $q(z|x , \phi)$ (with hidden units $h_1$) to infer the latent variables $z$.
    The loss $\ell$ is given by the Evidence lower bound (ELBO) computed from the likelihood $p(x | z, \theta)$ (with hidden units $h_2$). 
    We show two variants: the not-reparameterized variant (left), where the ELBO is given by $\E_x \left [ \E_z \left [ \log p(x | z, \theta) + \log p(z; \theta) \right ] - H[ q(\cdot | x, \phi)\right]$; and the reparameterized variant (right), where the ELBO is expressed as $\E_x \left [ \E_\epsilon \left [ \log p(x | h(x, \epsilon, \phi), \theta) + \log p( h(x, \epsilon, \phi), \theta) \right ] - H[ q(\cdot | x, \phi) ]\right]$.
    Note that $z$ is a deterministic node in the reparametrized graph. 
    (c) {\bf Markov decision process }. The objective here is the (undiscounted) return $\E_\tau [ \sum_t r(s_t) ]$ and trajectories $\tau=(\ldots, s_t, a_t, \ldots)$ are drawn from the composition of a policy $\pi_\theta$ and the system dynamics: $\tau \sim p(\tau, \theta) = p(s_0) \prod_t \pi(a_t|s_t, \theta) p(s_{t+1}|s_t, a_t)$.
    \label{fig:example_models}
    }
\end{minipage}

\paragraph{Likelihood ratio estimator}

For a random variable $x$ parameterized by $\theta$, i.e.\ $x \sim p(\cdot|\theta)$ the gradient of an expectation can be obtained using the following estimator:
\begin{align}
    \DxbyDy{}{\theta}\E_x [ f(x) ] = \E_x \left[ \DxbyDy{}{\theta} \log p(x|\theta) f(x) \right]
\end{align}

This classical result from the literature is known under a variety of names, including likelihood ratio estimator, or ``REINFORCE'' and can readily be derived by noting that $\DxbyDy{}{z} \log g (z) = \frac{1}{g(z)} \DxbyDy{}{z} g(z)$ for any function $g$. The quantity $\DxbyDy{}{\theta} \log p(x|\theta)$ is known as the \emph{score function} of variable $x$.

\paragraph{Pathwise derivative estimator}

In many cases, a random variable $x \sim p_\theta$ can be rewritten as a differentiable, deterministic function $x(\epsilon, \theta)$ of a fixed random variable with parameterless distribution $\epsilon \sim p_\epsilon$\footnote{Note that reparametrization is always possible, but differentiable reparametrization is not.}. This leads to a new expectation for which $\theta$ now only appears inside the expectation and the gradient can be straightforwardly estimated:
\begin{align}
   \DxbyDy{}{\theta} \E_\epsilon [ f(x(\epsilon; \theta)) ] = \E_\epsilon \left[ \DxbyDy{}{\theta}f(x(\epsilon, \theta)) \right] = \E_\epsilon \left[ \dxbydy{f}{x}\dxbydy{x(\epsilon, \theta)}{\theta} \right]
\end{align}

Both estimators remain applicable when $f$ is a function of $\theta$ itself. In particular, for the score function estimator we obtain
\begin{align}
    \DxbyDy{}{\theta} \E_x [ f(x,\theta) ] = \E_x \left[ \left(\dxbydy{}{\theta} \log p_x(x | \theta)\right) f(x,\theta) + \dxbydy{f}{\theta}\right]
\end{align}
and for the reparametrization approach, we obtain:
\begin{align}
   \DxbyDy{}{\theta} \E_\epsilon [ f(x(\epsilon, \theta)) ] = \E_\epsilon \left[\dxbydy{f}{x}\dxbydy{x}{\theta}+\dxbydy{f}{\theta} \right].
\end{align}

Other estimators have been introduced recently, relying on the implicit function theorem \citep{figurnov2018implicit}, continuous approximations to discrete sampling using Gumbel function \citep{maddison2016concrete,jang2016categorical}, law of large numbers applied to sums of discrete samples \citep{bengio2013estimating}, and others.

\subsection{Stochastic computation graphs}

\label{sec:SCG}

\paragraph{SCG framework.} We quickly recall the main results from \citet{schulman2015gradient}. 
\begin{mdframed}
\begin{definition}[Stochastic Computation Graph]
A stochastic computation graph $\cG$ is a directed, acyclic graph , with two classes of nodes (also called variables): deterministic nodes, and stochastic nodes.
\begin{enumerate}
\item \textbf{Stochastic nodes}, (represented with circles, denoted $\cS$), which are distributed conditionally given their parents. 
\item \textbf{Deterministic nodes} (represented with squares, denoted $\cD$), which are deterministic functions of their parents.
\end{enumerate}
We further specialize our notion of deterministic nodes as follows:
\begin{itemize}
\item Certain deterministic nodes with no parents in the graphs are referred to as \textbf{input nodes} $\theta \in \cInp$, which are set externally, including the parameters we differentiate with respect to. 
\item Certain deterministic nodes are designated as \textbf{losses} or \textbf{costs} (represented with diamonds) and denoted $\ell \in \cL$ . We aim to compute the gradient of the expected sum of costs with respect to some input node $\theta$. 
\end{itemize}
\vspace{-0.3cm}
A parent $v$ of a node $w$ is connected to it by a directed edge $(v, w)$. 
Let $L=\sum_{\ell\in\cL} \ell$ be the total cost. \label{def:cg}
\end{definition}
\end{mdframed}

For a node $v$, we let $h_v$ denote the set of parents of $v$ in the graph. 
A path between $v$ and $w$ is a sequence of nodes $a_0=v, a_1,\ldots, a_{K-1}, a_K=w$  such that all $(a_{k-1},a_k)$ are directed edges in the graph; if any such path exists, we say that $w$ descends from $v$ and denote it $v\prec w$.
We say the path is blocked by a set of variables $\cV$ if any of the $a_i\in \cV$. By convention, $v$ descends from itself.
For a variable $x$ and set $\cV$, we say $x$ can be deterministically computed from $\cV$ if there is no path from a stochastic variable to $x$ which is not blocked by $\cV$ (cf. Fig.~\ref{fig:deterministic-computable}). Algorithmically, this means that knowing the values of all variables in $\cV$ allows to compute $x$ without sampling any other variables; mathematically, it means that conditional on $\cV$, $x$ is a constant. 
Finally, whenever we use the notion of conditional independence, we refer to the notion of conditional independence informed by the graph (i.e.\ d-separation) \citep{geiger1990identifying,koller2009probabilistic}.

\paragraph{Gradient estimator for SCG.}

Consider the expected loss  $\cJ(\theta)=\E_{s \in \cS}\left[L \right]$. 
We present the general gradient estimator for the gradient of the expected loss $\DbyDt \cJ(\theta)$ derived in \citep{schulman2015gradient}.

For any stochastic variable $v$, we let $\log p(v)$ denote the conditional log-probability of $v$ given its parents, i.\ e.\ the value $\log p(v|h_v)$\footnote{Note $\log p(v)$ is indeed the conditional distribution, not the marginal one. The parents $h_v$ of $v$ are implicit in this notation, by analogy with deterministic layers, whose parents are typically not explicitly written out.}, and let $s(v,\theta)$ denote the score function $\DxbyDy{\log p(v)}{\theta}$.

\begin{theorem}\label{thm:main}[Theorem 1 from \citep{schulman2015gradient}\footnote{Since we defined the gradient of sampling operations to be zero, we do not need to use the notion of of deterministic descendence as in the original theorem; as the gradient of non-deterministic descendents with respect to inputs is always zero.}]
\textit{Under simple regularity conditions,}
  \begin{align*}
    \nonumber  \DxbyDy{}{\theta}\cJ(\theta)
    = \E \Bigg[
    \sum_{\substack{ v\in \cS \\\theta \prec v}}
    s(v,\theta) {L} 
     + \sum_{\substack{\ell \in \cL\\ \theta \prec \ell}} \DxbyDy{\ell}{\theta}
    \Bigg]
  \end{align*}
\end{theorem} \tabularnewline

Here, the first term corresponds to the influence $\inp$ has on the loss through the non-differentiable path mediated by stochastic nodes.
Intuitively, when using this estimator for gradient descent, $\inp$ is changed so as to increase or `reinforce' the probability of samples of $v$ that empirically led to lower total cost $L$. The second term corresponds to the direct influence $\inp$ has on the total cost through differentiable paths. Note that differentiable paths include paths going through reparameterized random variables.

\begin{figure}[H]
    \begin{minipage}[c]{0.4\textwidth}
        \includegraphics[width=1\textwidth]{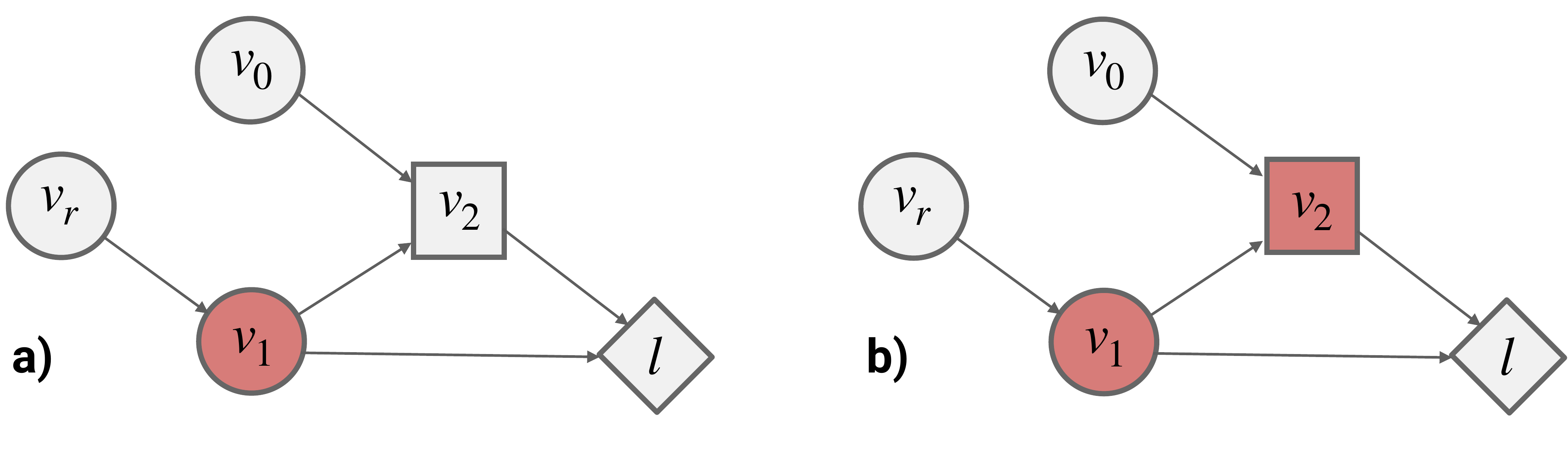}
    \end{minipage}\hfill
    \begin{minipage}[c]{0.55\textwidth}
        \label{fig:blocked}
        \caption{
        \emph{Blocked paths and deterministic computability.}
        a) $\cX=\lbrace v_1\rbrace$, shown in red, blocks the path from $v_r$ to $\ell$, but not $v_0$ to $\ell$.
        b) $\ell$ can be deterministically computed from $\cX=\{v_1,v_2\}$ (red) because all stochastic variables have their path to $L$ blocked by $\cX$ ($v_0$, blocked by $v_2$; $v_r$, blocked by $v_1$; and $v_1$, blocked by itself). 
        }   \label{fig:deterministic-computable}
\end{minipage}

\end{figure}

\section{Value based methods}
\label{sec:valueCritic}

The gradient estimator from theorem \ref{thm:main} is very general and conceptually simple but it tends to have high variance \citep[see for instance the analysis found in][]{mnih2016variational}, which affects convergence speed \citep[see e.g.][]{schmidt2011convergence}. 
Furthermore, it requires a full evaluation of the graph.
To address these issues,
we first discuss several variations of the basic estimator in which the total cost $L$ is replaced by its \emph{deviation} from the expected total cost, or conditional expectations thereof, with the aim of reducing the variance of the estimator. 
We then discuss how approximations of these conditional expectations can be learned locally, leading to a scheme in which gradient computations from partial model evaluations become possible.

\subsection{Values}\label{subsec:Values}

In this section, we use the simple concept of conditional expectations to introduce a general definition of value function in a stochastic computation graph.

\begin{definition}[Value function]
Let $\cX$ be an arbitrary subset of $\cG$, $\mathbf{x}$ an assignment of possible values to variables in $\cX$ and $S$ an arbitrary scalar value in the graph.
The value function for set $\cX$ is the expectation of the quantity $S$ conditioned on $\cX$:
\begin{align*}
 V: \mathbf{x}\mapsto V(\mathbf{x}; S) = \E_{\cG\setminus \cX|\cX=\mathbf{x}}\left[S\right].
\end{align*}
\label{def:value}
\end{definition}

Intuitively, a value function is an estimate of the cost which averages out the effect of stochastic variables not in $\cX$,
therefore the larger the set, the fewer variables are averaged out. 

The definition of the value function as conditional expectation results in the following characterization: 
\begin{lemma}\label{corr:sq-min}
For a given  assignment $\mathbf{x}$ of $\cX$,  $V(\mathbf x; S)$ is the optimal mean-squared error estimator of $S$ given input $\cX$:
\begin{align*}
   V(\mathbf{x};S) = \operatorname{argmin}_{v_{\mathbf x}} \E_{\cG\setminus \cX|\cX=\mathbf x}\left[(S-v_{\mathbf x})^2\right].
\end{align*}
\end{lemma}

Consider an arbitrary node $v \in \cG$, and let $L(v) = \sum_{\substack{\ell \in \cL\\ v \prec \ell}} \ell$ denote the $v$-rooted cost-to-go, i.e.\ the sum of costs `downstreams' from $v$ (similar notation is used for $L(V)$ if $V$ is a set). The scalar $S$ will often be the cost-to-go $L(v)$ for some fixed node $v$; furthermore, when clear from context, we use $\cX$ to both refer to the variables and the values they take. For notational simplicity, we will denote the corresponding value function $V(\cX)$.

Fig.~\ref{fig:value_examples} shows multiple examples of value functions for different graphs. 
The above definition is broader than the typical one used in reinforcement learning. There, due to the simple chain structure of the Markov Decision Processes, the resulting Markov properties of the graph, and the particular choice of $\cX$, the expectation is only with respect to downstream nodes. Importantly, according to Def.\ \ref{def:value} the value can depend on $\cX$ via ancestors of $\cX$ (e.g.\ example in Fig. \ref{fig:value_examples}c). Lemma \ref{corr:sq-min} remains valid nevertheless.

\subsection{Baselines and critics}
\label{sec:value:baseline_critic}

In this section, we will define the notions of baselines and critics and use them to introduce a generalization of theorem~\ref{thm:main} which can be used to compute lower variance estimator of the gradient of the expected cost. We will then show how to use value functions to design baselines and critics.

Consider an arbitrary node $v$ and input $\theta$.
\begin{definition}[Baseline]\label{def:Baseline}
A \emph{baseline} $B$ for $v$ is any function of the graph such that $\E[s(v,\theta) B]=0$.
A \emph{baseline set} $\cB$ is an arbitrary subset of the non-descendants of $v$.
\end{definition}
\vspace{-0.2cm}
Baseline sets are of interest because of the following property:
\begin{property}
Let $B$ be an arbitrary scalar function of $\cB$.  Then $B$ is a baseline for $v$.
\end{property}
\vspace{-0.2cm}
Common choices are constant baselines, i.e.\ $\cB=\emptyset$, or baselines $B(h_v)$ only depending on the parents $\cB=h_v$ of $v$. 

\begin{definition}[Critic]
A \emph{critic} $Q$ of cost $L(v)$ for $v$ is any function of the graph such that $\E[s(v,\theta)(L(v)-Q)]=0$.
\label{def:critic}
\end{definition}
\vspace{-0.2cm}

By linearity of expectations, linear combinations of baselines are baselines, and convex combinations of critics are critics.

The use of the terms \emph{critic} and \emph{baseline} is motivated by their respective roles in the following theorem, which generalizes the policy gradient theorem \citep{sutton2000policy}:
\begin{theorem}\label{thm:value-critic}
Consider an arbitrary baseline $B_v$ and critic $Q_v$ for each stochastic node $v$. Then,
  \begin{align*}
    \nonumber \DbyDt \cJ(\theta)
    = \E \Bigg[
    &\sum_{\substack{v \in \cS\\ \inp \prec  v}}
    G_v + \sum_{\substack{\ell \in \cL\\ \inp \prec \ell}} \DbyDt \ell
    \Bigg],
  \end{align*}
  where $$G_v = s(v,\theta)  \Big(Q_v-B_v\Big).$$
\end{theorem}
The difference $Q_v-B_v$ between a critic and a baseline is called an \emph{advantage} function.

Theorem \ref{thm:value-critic}  enables the derivation of a surrogate loss. Let $L^s$ be defined as $L^s=L+\sum_{\substack{\ell \in \cL\\ \inp \prec \ell}}
\log p(v) \left(\sg{Q}-\sg{B}\right)$, where we recall that the tilde notation indicates a constant from the point of view of computing gradients.
Then, the gradient of the expected cost $\cJ(\theta)$ equals the gradient of $L^s$ in expectation: $\DxbyDy{}{\theta}\E[L]=\E\left[\DxbyDy{}{\theta}L^s\right]$.

Before providing intuition on this theorem, we see how value functions can be used to design baselines and critics:

\begin{definition}[Baseline value function and critic value function]\label{def:BCVF}\ \\
For any node $v$ and baseline set $\cB$, a special case of a baseline is to choose the value function with set $\cB$. Such a baseline is called a \textbf{baseline value function}. \\
Let a critic set $\cC$ be a set such that $v\in \cC$, and ${\log p(v)}$ and $L(v)$ are conditionally independent given $\cC$; a special case is when $\cC$ is such that $\log p(v)$ is deterministically computable given $\cC$.
Then the value function for set $\cC$ is a critic for $v$ which we call a \textbf{critic value function} for $v$. 
\end{definition}

In the standard MDP setup of the RL literature, $\cC$ consists of the state $s$ and the action $a$ which is taken by a stochastic policy $\pi$ in state $s$ with probability $\log \pi(a|s)$, which is a deterministic function of $(s,a)$.
Definition \ref{def:BCVF} is more general than this conventional usage of critics since it does not require $\cC$ to contain all stochastic ancestor nodes that are required to evaluate $\log p(v)$. For instance, assume that the action is conditionally sampled from the state $s$ and some source of noise $\xi$, for instance due to dropout, with distribution $\pi(a|\xi, s)$\footnote{in this example, it is important $\xi$ is used only once; it cannot be used to compute other actions.}. The critic set may but does not need to include $\xi$; if it does not, $\log \pi(a|\xi,s)$ is not a deterministic function of $a$ and $s$. The corresponding critic remains useful and valid.

Figure \ref{fig:value_examples} contains several examples of value functions which take the role of baselines and critics for different nodes.

\begin{figure}
    \centering
    \includegraphics[width=0.90\textwidth]{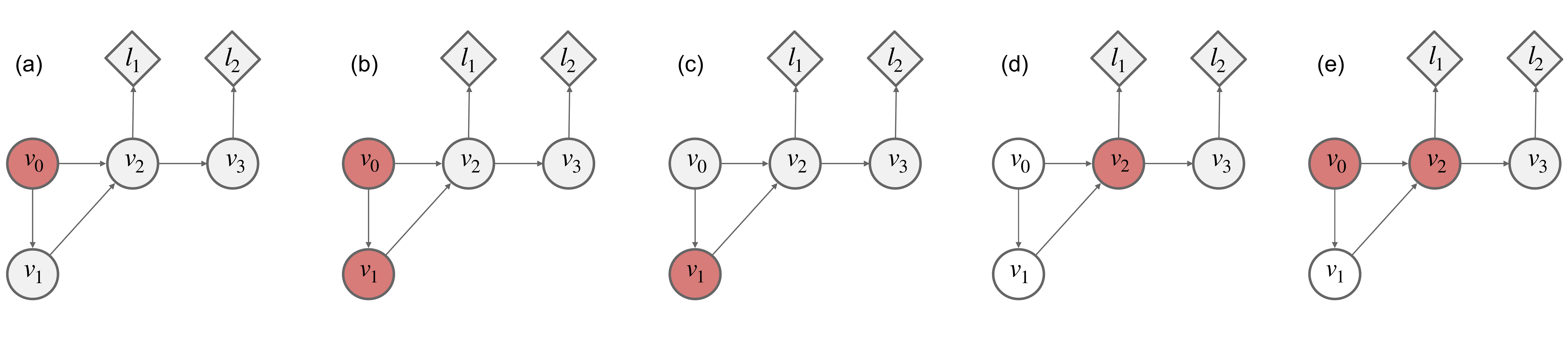}
    \caption{\emph{Examples of baseline value functions and critic value functions.}
    (a-e) Different value functions for the same SCG with $\cX^{(a)}=\{v_0\}$, $\cX^{(b)}=\{v_0, v_1\}$, $\cX^{(c)}=\{v_1\}$, $\cX^{(d)}=\{v_2\}$, $\cX^{(e)}=\{v_0, v_2\}$. The variables averaged over are shaded in light gray. The value function in (c) involves a marginalization over `upstream' variables: $V(v_1) =  \E_{v_0,v_2,v_3 | v_1}[\ell_1 + \ell_3]$.\\
    (a,b,c) are valid baselines for $v_2$; (d,e) cannot act as baselines for $v_2$ since $v_2$ belongs to those sets, but can act as baselines for $v_3$.\\
    (a,b,e) are critics for $v_0$, but (c,d) are not. (b) is a critic for $v_1$. 
    (c) is not a critic for $v_1$ since $L(v_1)$ and $\log \pi_{v_1}$ are correlated conditionally on $\cX^{(c)}$ (through $v_0$).
    Finally (d,e) are critics for $v_2$; note however $\log p(v_2)$ is not a deterministic function of either $\cX^{(d)}$ or $\cX^{(e)}$. 
    \label{fig:value_examples}}
\end{figure}

Three related ideas guide the derivation of theorem \ref{thm:value-critic}.
To give intuition, let us analyze the term $G_v$, which replaces the score function weighted by the total cost $s(v,\theta) L$.
First, the conditional distribution of $v$ only influences the costs downstream from $v$, 
hence we only have to reinforce the probability of $v$ with the cost-to-go $L(v)$ instead the total cost $L$.
Second, the extent to which a particular sample $v$ contributed to cost-to-go $L(v)$ should be compared to the cost-to-go the graph typically produces in the first place. This is the intuition behind subtracting the baseline $B$, also known as a \emph{control variate}. 
Third, we ideally would like to understand the precise contribution of $v$ to the cost-to-go, not for a particular value of downstream random variables, but on average. This is the idea behind the critic $Q$.
The advantage (difference between critic and baseline) therefore provides an estimate of `how much better than anticipated' the cost was, as a function of the random choice $v$. 

Baseline value functions are often used as baselines as they approximate the optimal baseline (see Appendix~\ref{sec:opt-baseline}). Critic value functions are often used as they provide an expected downstream cost given the conditioning set. Furthermore, as we will see in the next section, value functions can be estimated in a recursive fashion, enabling local learning of the values, and sharing of value functions between baselines and critics.
For these reasons, in the rest of this paper, we will only consider baseline value functions and critic value functions.

\emph{In the remainder of this section, we consider an arbitrary value function with conditioning set $\cX$.}

\subsection{Recursive estimation and Markov properties}
\label{sec:valueCritic:boostrapping}

A fundamental principle in RL is given by the Bellman equation -- which details how a value function can be defined recursively in terms of the value function at the next time step.  In this section, we generalize the notion of recursive computation to arbitrary graphs. 

The main result, which follows immediately from the law of iterated expectations, characterizes the value function for one set, as an expectation of a value function (or critic / baseline value function) of a larger set:
\begin{lemma}\label{lem:average}
Consider two sets $\cX^1 \subset \cX^2$, and an arbitrary quantity $S$. Then we have: $\E[V(\cX^2; S)|\cX^1]=V(\cX^1; S)$.
\end{lemma}

This lemma is powerful, as it allows to relate value functions as average of over value function. A simple example in RL is the relation (here, in the infinite discounted case) between the Q function $Q^\pi(s,a)=\E[R|s,a]$ of a policy and the corresponding value function $V^\pi(s)=\E[R|s]$, which is given by $V^\pi(s)=\sum_a \pi(a|s) Q^\pi(s,a)$. Note this equation relates a critic value function to a value function typically used as baseline.

To fully leverage the lemma above, we proceed with a Markov property for graphs\footnote{borrowed from well known conditional independence conditions in graphical models, and adapted to our purposes.}, which captures the following situation: given two conditioning sets $\cX^1\subset \cX^2$, it may be the case that the additional information contained in $\cX^2$ does not improve the accuracy of the cost prediction compared to the information contained in the smaller set $\cX^1$.

\begin{definition}\label{def:Markov}
For conditioning set $\cX$, we say that $\cX$ is \textbf{Markov} (for $L(v)$) if for any $w$ such that there exists a directed path from $w$ to $L(v)$ not blocked by $\cX$, none of the descendants of $w$ are in $\cX$.  
\end{definition}

Let $\cX^\uparrow$ be the set of all ancestors of nodes $\cX$\footnote{Recall that by convention nodes are descendants of themselves, so $\cX \subset \cX^\uparrow$}.
\begin{property}\label{prop:Markov}
Let $\cX$ be Markov, consider any $\cX'$ such that $\cX\subset \cX'\subset \cX^\uparrow$. For any $x'$ assignment of values to the variables in $\cX'$, let $x'_{|\cX}$ be the restriction of $x'$ to the variables in $\cX$. Then:
$$\forall x', V(\cX=x'_{|\cX})=V(\cX'=x'),$$
which we will simply denote, with a slight abuse of notation, $$V(\cX')=V(\cX).$$
\end{property}

In other words, the information contained in $\cX^\uparrow\setminus \cX$ is irrelevant in terms of cost prediction, given access to the information in $\cX$. Several examples are shown in Fig.~\ref{fig:markov_examples}. It is worth noting that Def.\ \ref{def:Markov} does not rule out changes in the expected value of $L(v)$ after adding additional nodes to $\cX$ (cf.\ Fig.~\ref{fig:markov_examples}(d,e)). Instead it rules out correlations between $\cX$ and $L(v)$ that are mediated via \emph{ancestors} of nodes in $\cX$ as in the example in Fig.~\ref{fig:markov_examples}(a,b,c)).

\begin{figure}
    \centering
    \includegraphics[width=.9\textwidth]{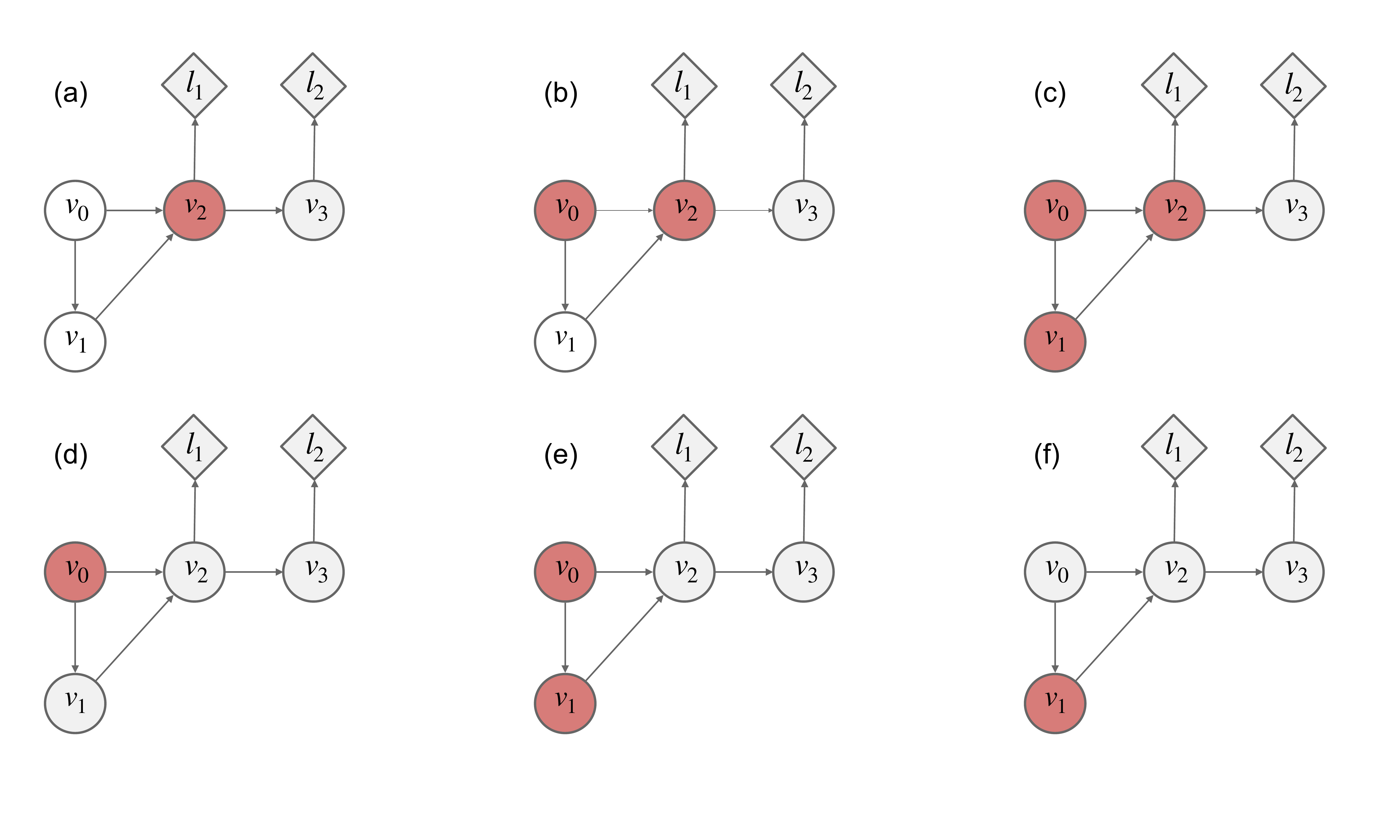}
    \caption{\emph{Markov property for value functions.}
    Consider node $v_2$ and corresponding cost $L(v_2)=\ell_1+\ell_2$; we consider the Markov property for $L(v_2)$.
    For (a,b,c), $\cX^{(a)} \subset \cX^{(b)} \subset \cX^{(c)}=\cX^{(a)\uparrow}$ and the Markov property holds since no ancestor of $\cX$ has an unblocked path to $L(v_2)$. This implies in particular it is sufficient to condition on $\cX^{(a)}$;
    all associated value functions compute the same value, i.e.\ $V(\cX^{(a)}) = V(\cX^{(b)}) = V(\cX^{(c)})$.
    The same applies to (d,e); note however than $V(\cX^{(d)}) \neq V(\cX^{(e)})$, though both sets are Markov as well. This is because neither $\cX^{(d)} \subset \cX^{(e)}\subset \cX^{(d)\uparrow} $ nor $\cX^{(e)}\subset \cX^{(d)}\subset \cX^{(e)\uparrow} $ hold.
    While the set inclusion property holds for (e,f), i.\ e.\ $\cX^{(f)} \subset \cX^{(e)}\subset \cX^{(f)\uparrow}$, we actually have $V(\cX^{(e)}) \neq V(\cX^{(f)})$. This is because $\cX^{(f)}$ is not Markov, which leads to an implicit marginalization of $v_0$ ($p(v_2|v_0,v_1)\neq \int_{v_0} p(v_0|v_1) p(v_2|v_1,v_0)$).
    \label{fig:markov_examples}}
\end{figure}

The notion of Markov set can be used to refine Lemma~\ref{lem:average} as follows:
\begin{lemma}[Generalized Bellman equation]\label{lem:Bellman}
Consider two sets $\cX^1 \subset \cX^{2\uparrow}$, and suppose $\cX^2$ is Markov. Then we have: $\E[V(\cX^2)|\cX^1]=V(\cX^1)$.
\end{lemma}
The Markov assumption is critical in allowing to `push' the boundary at which the expectation is defined; without it, lemma~\ref{lem:average} only allows to relate value functions of sets which are subset of one another. But notice here that no such inclusion is required between $\cX^1$ and $\cX^2$ themselves. In the context of RL, this corresponds to equations of the type $V(s)=\sum_a \pi(s,a)\left(r(s,a)+ \sum_{s'} P(s'|s,a) V(s')\right)$ (see Fig.~\ref{fig:bellman-eq-RL}), though to get the separation between the reward and the value at the next time step, we need a slight refinement, which we detail in the next section.

We next go through a few examples of applying Lemmas ~\ref{lem:average} and ~\ref{lem:Bellman}.\clearpage
\begin{example}
In Fig.~\ref{fig:markov_examples}, we have $\cX^{(d)}\subset\cX^{(a)^\uparrow}$, and $\cX^{(a)}$ is Markov, from which we obtain $\E[V(v_2)|v_0]=V(v_0)$. We also have $\cX^{(d)}\subset\cX^{(f)^\uparrow}$, but because $\cX^{(f)}$ is not Markov, $\E[V(v_1)|v_0]\not = V(v_0)$.
On the other hand, because $\cX^{(a)}$ and $\cX^{(f)}$ are both subset of $\cX^{(c)}$ (this is in fact true of all sets in Fig.~\ref{fig:markov_examples}), so we have, regardless of Markov assumptions, $V(v_2)=\E[V(v_0,v_1,v_2)|v_2]$ and $V(v_1)=\E[V(v_0,v_1,v_2)|v_1]$.
\end{example}

\begin{figure}
    \begin{minipage}[c]{0.35\textwidth}
        \includegraphics[width=1\textwidth]{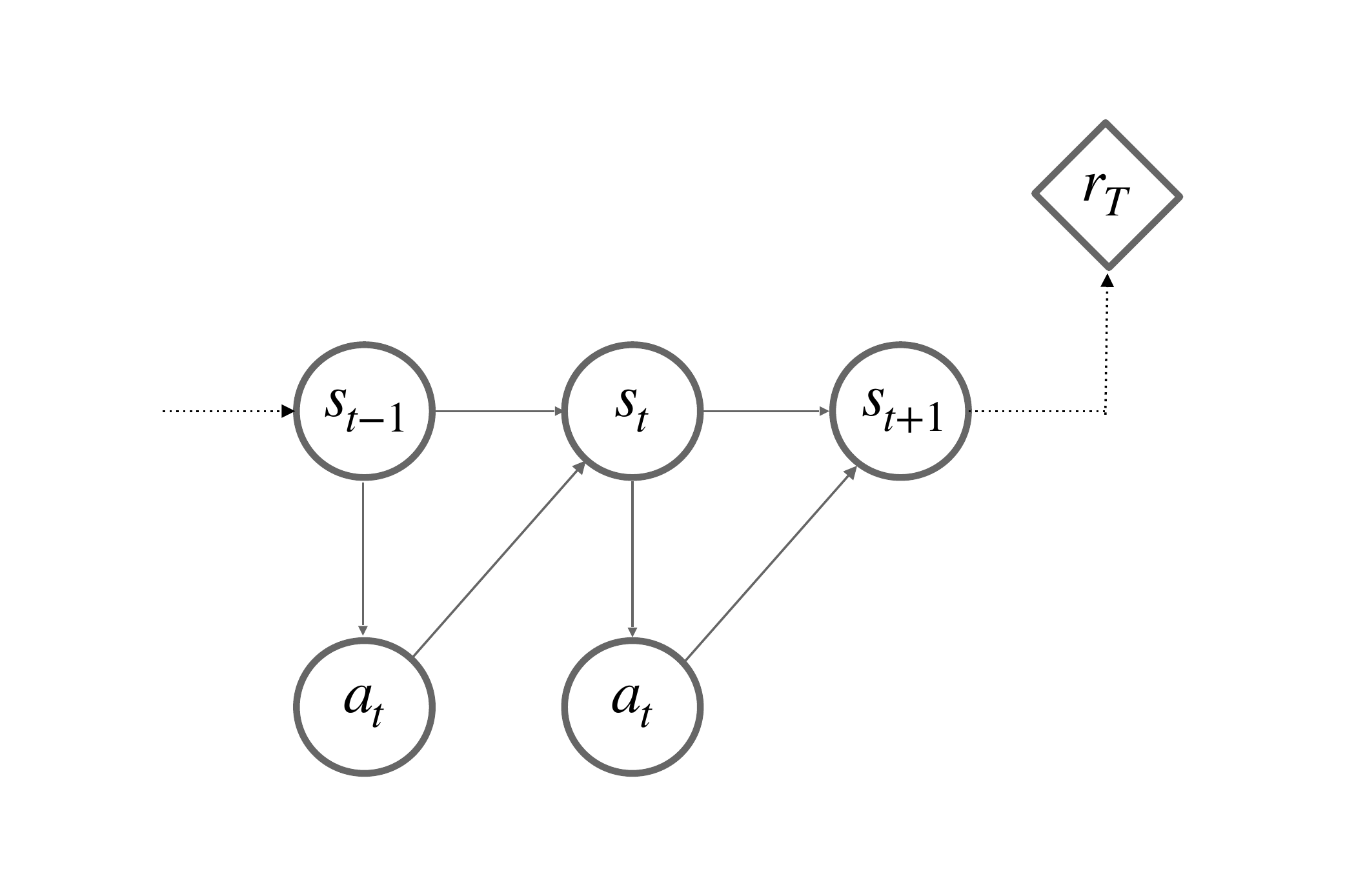}
    \end{minipage}\hfill
    \begin{minipage}[c]{0.65\textwidth}
        \captionof{figure}{
        \emph{Bellman equation in RL.}\\
        The Bellman equation for Markov Decision Processes with single terminal reward $r_T$. The Q-function is $Q(s_t, a_t)=\E[r_T|s_t, a_t]$, the value function $V(s_t)=\E[r_T|s_t]$. From Lemma~\ref{lem:average}, 
        $V(s_t)=\E[Q(s_t,a_t)|s_t]$. Since $s_{t+1}$ is Markov and $s_t\subset s_{t+1}^\uparrow=\lbrace s_j, a_{j-1}, j \leq t+1\rbrace $, from Lemma~\ref{lem:Bellman}, we have $V(s_t)=\E[V(s_{t+1})|s_t]$.
        }
        \label{fig:bellman-eq-RL}
\end{minipage}
\end{figure}

\subsection{Decomposed costs and bootstrap}

In the previous sections we have considered a value function with respect to a node $v$ which predicts an estimate of the cost-to-go $L(v)$ from node $v$ (note $L(v)$ was implicit in most of our notation). In this section, we write the cost-to-go at a node as a funtion of cost-to-go from other nodes or collection of nodes, and leverage the linearity of expectation to turn these relations between costs into relation between value functions.

A first simple observation is that because of the linearity of expectations, for any two scalar quantities $S_1,S_2$, real value $\lambda$ and set $\cX$, we have $V(\cX, \lambda S_1+S_2)=\lambda V(\cX, S_1)+ V(\cX, S_2)$. 

\begin{definition}[Decomposed costs]
For a node $v$ and a collection $\mathbf{V}=\lbrace V_0,V_1,\ldots, V_D\rbrace$ in the graph, we say that the cost $L(v)$ can be decomposed with set $\mathbf{V}$ if $L(v)=\sum_i L(V_i)$.
\end{definition}

This implies that cost nodes can be grouped in disjoint sets corresponding to the descendents of different sets $V_i$, without double-counting.
A common special case is a tree, where each $V_i$ is a singleton containing a single child $\{v_i\}$ of $v$.

\begin{theorem}[Bootstrap principle for SCGs]\label{thm:bootstrap}
Suppose the cost-to-go $L(v)$ from node $v$ can be decomposed with sets $\mathbf{V}=\lbrace V_0, \ldots, V_D\rbrace$, and consider an arbitrary set $\cX_v$ with associated value function $V(\cX_v,L(v))$.
Furthermore, for each set $V_i$, consider a set $\cX_{V_i}$ and associated value function: $V(\cX_{V_i}, L(V_i))$.
If for each $i$, $\cX_v\subset \cX_{V_i}$, or if for each $i$, $\cX_{V_i}$ is Markov and $\cX_v \subset \cX_{V_i}^\uparrow$, then:
\begin{align*}
 V(\cX_v) = \sum_i \E_{\cG \setminus \cX_v|\cX_v}\left[ V(\cX_{V_i})
 \right].
\end{align*}
\label{theorem:boostrap}
\end{theorem}
Fig.~\ref{fig:bootstrapping_examples} highlights potential difficulties of defining correct bootstrap equations for various graphs.

\begin{figure}[h]
    \centering
    \includegraphics[width=0.8\textwidth]{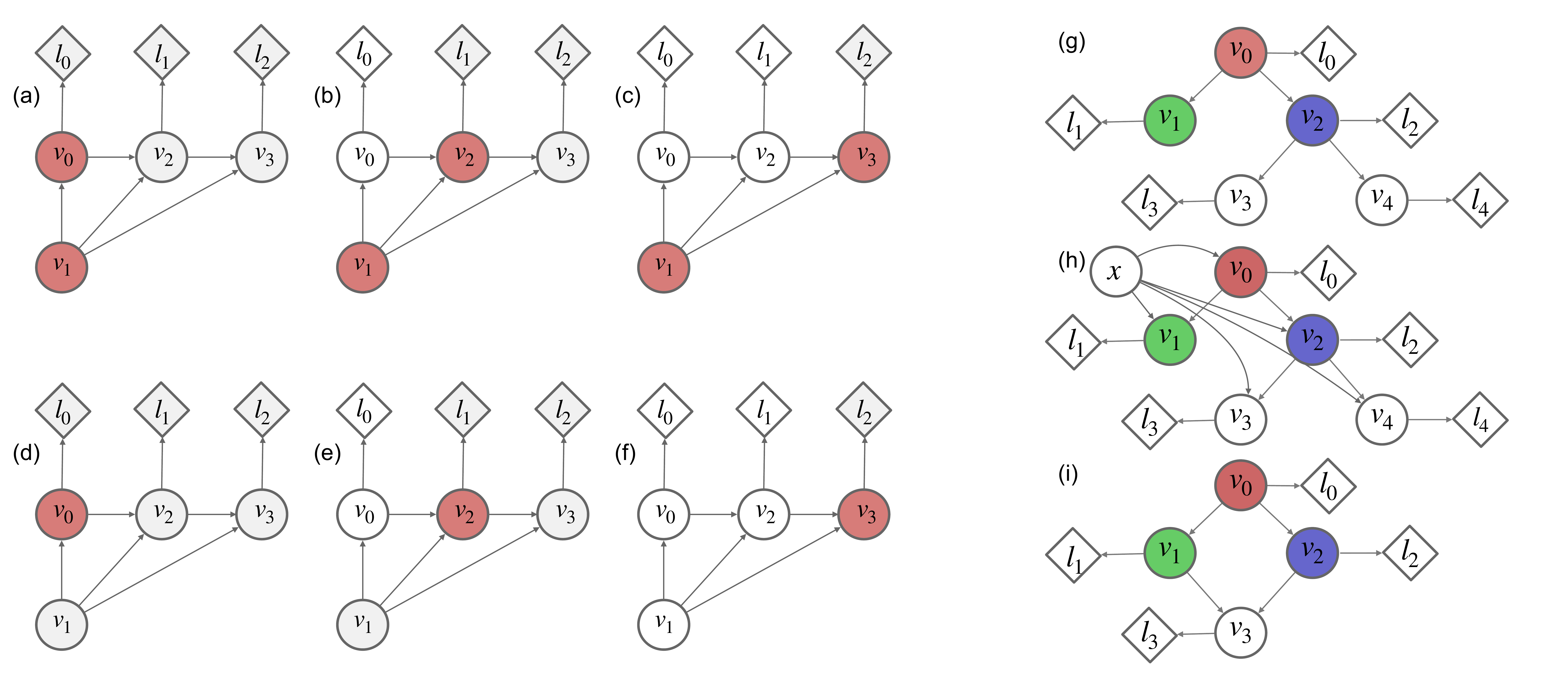}
    \caption{\emph{Bootstrapping.}
    The value functions in (a,b,c) and  (d,e,f) illustrate, for the same graph, ensembles of value functions that do and do not allow bootstrapping respectively. For (a,b,c) we have $V(v_0,v_1) = \E_{v_2,v_3 | v_0,v_1} [ \ell_1 + \ell_2 ] + \ell_0 = 
    \E_{v_2 | v_0,v_1} [ \E_{v_3 | v_2 , v_1 }[\ell_2] + \ell_1] + \ell_0 = 
    \E_{v_2 | v_0,v_1} [ \E_{v_3 | v_2 , v_1 }[ V(v_3) ]+ \ell_1] + \ell_0 = \E_{v_2 | v_0,v_1} [ V(v_1,v_2)] + \ell_0$. In (d,e,f) a Markov chain is conditioned on an additional shared parent $v_1$. Here,
    $V(v_0) = \E_{v_1,v_2,v_3 | v_0} [ \ell_1 + \ell_2 ] + \ell_0$ cannot be expressed in terms of $V(v_2) = \E_{v_1,v_3 | v_2} [ \ell_2 ] + \ell_1 = \E_{v_1,v_3 | v_2} [ V(3) ] + \ell_1$ due to the implicit marginalization over the shared parent in each value function. Note, however, that $V(v_0)$ can be expressed in terms of $V(v_0, v_2)$. This is akin to a POMDP that can be translated into a MDP by treating the entire history as the state (cf.\ Section \ref{sec:Appendix:Examples:MDP}). (g,h,i) provide additional examples for three related graphs: In (g), thanks to Markov properties, value functions at parent nodes can be expressed in terms of the value functions of the children, e.g.\ $V(v_0) = \ell_0 + \E_{v_1 | v_0}[V(v_1)] + \E_{v_2 | v_0}[V(v_2)]$. In (h), simple Markov properties are missing, so bootstrap is possible only when value functions are additionally conditioned on $x$, i.e.\ $V(v_0,x) = \ell_0 + \E_{v_1 | v_0,x}[V(v_1)] + \E_{v_2 | v_0,x}[V(v_2,x)]$, or when the $\cX_v$ form supersets of each other $V(v_0) = \ell_0 + \E_{v_1 | v_0}[V(v_1,v_0)] + \E_{v_2 | v_0}[V(v_2,v_0)]$. But it is not possible to express e.g.\ $V(v_0)$ in terms of $V(v_2)$. In (i) the cost does not decompose if value functions $V(v_0),V(v_1),V(v_2)$ are naively defined in terms of all downstream costs.
    \label{fig:bootstrapping_examples}}
\end{figure}
From the bootstrap equation follows a special case, which we call \emph{partial averaging}, often used for critics: 
\begin{corollary}[Partial averages]\label{corr:partial-avg}
Suppose that for each $i$, $\cX_{V_i}$ is Markov and $\cX_{V_i}\subset \cX_v\subset \cX_{V_i}^\uparrow$. Without loss of generality, define $V_0$ as the collection of all cost nodes which can be deterministically computed from $\cX_v$. Then,
\begin{align*}
 V(\cX_v) = \sum_{\ell \in V_0} \ell+\sum_{i\geq 1} V(\cX_{V_i}).
\end{align*}
\end{corollary}

The term `partial average' indicates that the value function is a conditional expectation (i.e.\ 'averaging' variables) but that it combines averaged cost estimates (the value terms $V(\cX_{V_i})$) and empirical costs ($\sum_{\ell \in V_0} \ell_v$). 
 Fig.~\ref{fig:value_partial averages} shows some examples for generic graphs. 

In the case of RL for instance, a k-step return is a form of partial average, since the return $R_t$ -- sum of all rewards downstream from state $s_t$ -- can be written as the sum of all rewards in $V_0=\{r_{t}, \ldots, r_{t+k-1}\}$ and downstream from $V_1=\{s_{t+k}\}$; the critic value function $V(s_t, \ldots, s_{t+k})$ is therefore equal\footnote{We assume for simplicity that the rewards are deterministic functions of the state; the result can be trivially generalized.} to $\sum_{t'=t}^{t+k-1} r_{t'}+V(s_{t+k})$. This implies in turn that $V(s_t)=\E[V(s_t, \ldots, s_{t+k})]$ is also equal to $\E[\sum_{t'=t}^{t+k} r_{t'}+V(s_{t+k})]$.

\begin{figure}
    \begin{minipage}[c]{0.65\textwidth}
        \includegraphics[width=1\textwidth]{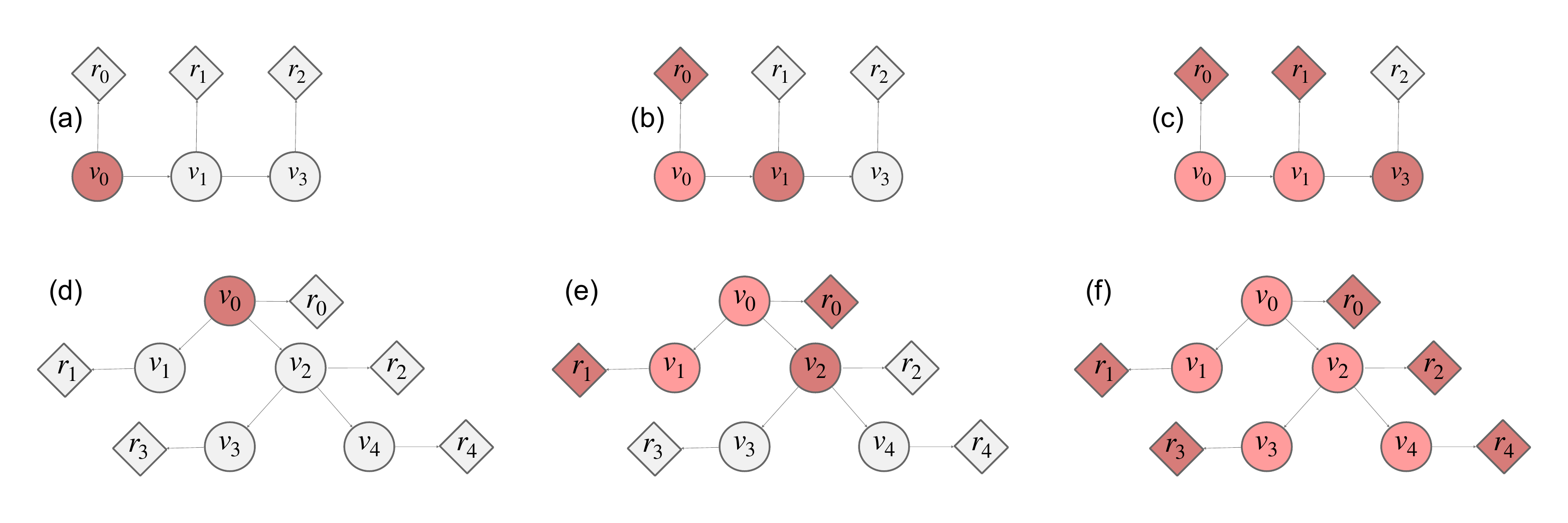}
    \end{minipage}\hfill
    \begin{minipage}[c]{0.35\textwidth}
        \captionof{figure}{
        \emph{Partial averages.}
        Examples of partial averages for a chain-structured and tree-structured graph respectively. In both cases the picture shows critics for node, computing the conditional expectations of $R=r_0+r_1+r_2$. Conditioning nodes are in red; dark red nodes contain information, but light nodes do not contain information not already contained in a dark node.
        The corresponding value functions are:
        a) $V(v_0)$, b) $r_0+V(v_1)$, c) $r_0+r_1+V(v_3)$, d) $V(v_0)$, e) $r_0+r_1+V(v_2)$, f) $r_0+r_1+r_2+r_3+r_4$.
        }
        \label{fig:value_partial averages}
\end{minipage}
\end{figure}

\subsection{Approximate Value functions}

\label{sec:comp-val-crit}
In practice, value functions often cannot be computed exactly.  In such cases, one can resort to learning parametric approximations.  For node $v$, conditioning set $\cX$, we will consider an approximate value function $\hV(\cX)$ as an approximation (with parameters $\phi$) to the value function $V(\cX)=\E_{\cG\setminus \cX_v|\cX_v}\left[L(v)\right]$.

Following corollary~\ref{corr:sq-min}, we know that for a possible assignment $\bx$ of variables $\cX$, $V(\bx)$ minimizes $\E_{\cG\setminus \cX|\cX}\left[(L(v)-v_{\bx})^2\right]$ over $v_{\bx}$. We therefore elect to optimize $\phi$ by considering the following weighted average, called a \emph{regression on return} in reinforcement learning:
\begin{eqnarray}
\mathcal{L}^\phi&=& \E_{\cX}\left[\E_{\cG\setminus \cX|\cX}\left[(L(v)-\hV(\cX))^2\right]\right]\nonumber\\ 
&=& \E_\cG\left[(L(v)-\hV(\cX))^2\right], \nonumber
\end{eqnarray}
from which we obtain (note that $\phi$ does not affect the distribution of any variable in the graph, and therefore exchange of derivative and integration follows under common regularity conditions):
\begin{eqnarray}
\DxbyDy {\mathcal{L}^\phi}{\phi} &= &\E_\cG\left[\DbyDp {\hV(\cX)} \left(\hV(\cX)-L(v)\right)\right], \label{eq:ROR}
\end{eqnarray}
which can easily be computed by forward sampling from $\cG$, even if conditional sampling given $\cX$ is difficult. This is possible because of the use of $\E_\cX$ as a particular weighting on the collection of problems of the type $\E_{\cG\setminus \cX|\cX}[(L(v)-\hV(\cX))^2]$. 

We now leverage the recursion methods from the previous sections in two different ways.
The first is to use the combination of approximate value functions and partial averages to define other value functions.
For a partial average as defined in theorem \ref{corr:partial-avg} and family of approximate value functions $\hV_{v_i}(\cX_{v_i})$, we can define an approximate value function through the bootstrap equation: $\sum_{\ell \in V_0} \ell_v+\sum_i \hV(\cX_{v_i})$. In other words, using the bootstrap equations, approximating value functions for certain sets automatically defines other approximate value functions for other sets.

In general, we can trade bias and variance by making $V_0$ larger (which will typically result in lower bias, higher variance) or not, i.e.\ by shifting the boundary at which variables are integrated out. 
An extreme case of a partial average is not an average at all, where $\cX=\cG$, in which case the value function is the empirical return $L(v)$. 
K-step returns in reinforcement learning (see section \ref{sec:Appendix:Examples:MDP}) are an example of trading bias and variance by choosing the integration boundary to be all nodes at a distance greater than $K$, and $V_0$ all costs at a distance less than $K$. $\lambda$-weighted returns in the RL literature (Section \ref{sec:Appendix:Examples:MDP}) are convex combinations of partial averages. $\lambda$ similarly controls a bias-variance tradeoff.

The second way to use the bootstrap equation is to provide a different, lower variance target to the value function regression. By combining theorem \ref{thm:bootstrap} and equation \ref{eq:ROR}, we obtain:
\begin{align}\label{eq:Bellman}
\DxbyDy {\mathcal{L}_\phi}{\phi}= \E_\cG\Bigg[\DbyDp &\hV(\cX) \Bigg(\hV(\cX)-
\left(\sum_{\ell \in V_0} \ell_v +\sum_i \hV(\cX_{v_i})\right)\Bigg)\Bigg].
\end{align}

By following this gradient, the value function $\hV(\cX_v)$ will tend towards the bootstrap value $\sum_{\ell \in V_0} \ell_v +\sum_i \hV(\cX_{v_i})$ instead of the return $L(v)$. 
Because the former has averaged out stochastic nodes, it is a lower variance target, and should in practice provide a stronger learning signal. Furthermore, as it can be evaluated as soon as $\cX_{v_i}$ is evaluated, it provides a local or online learning rule for the value at $v$; by this we mean the corresponding gradient update can be computed as soon as all sets $\cX_{v_i}$ are evaluated. 
In RL, this local learning property can be found in actor-critic schemes: when taking action $a_t$ in state $s_t$, as soon as the immediate reward $r_t$ is computed and next state $s_{t+1}$ is evaluated, the value function $V(s_t)$ (which is a baseline for $a_t$) can be regressed against low-variance target $r_t+V(s_{t+1})$ (which is also a critic for $a_t$), and the temporal difference error (or advantage) can be used to update the policy by following $\DxbyDy{\log \pi(a_t|s_t,\theta)}{\theta}(r_t+V(s_{t+1})-V(s_t))$.

\section{Gradient-based methods}
\label{sec:gradientCritic}

In the previous section, we developed techniques to lower the variance of the score-function terms $\E\left[\Big(\DbyDt \log p(v) \Big) \Big(Q(\cC)-B(\cB_v)\Big)\right]$ in the gradient estimate. This led to the construction of a surrogate loss $L^s$ which satisfies $\DxbyDy{}{\theta}\E\left[\cJ(\theta)\right]=\E\left[\DxbyDy{L^s}{\theta}\right]$.

In this section, we develop corresponding techniques to lower the variance estimates of the gradients of surrogate cost $\DxbyDy{}{\theta}L^s$. 
To this end, we will again make use of conditional expectations to partially average out variability from stochastic nodes.
This leads to the idea of a \emph{gradient-critic}, the equivalent of the value critic for gradient-based approaches. 

\subsection{Gradient-Critic}

\begin{definition}[Value-gradient] 
The \textbf{value-gradient} for $v$ with set $\cC$ is the following function of $\cC$: 
\begin{align*}
    g(\cC) =  \E_{\cG \setminus \cC|\cC}\left[ \DxbyDy{L^s}{v}\right].
\end{align*}
\end{definition}

Value-gradients are not directly useful in our derivations but we will see later that certain value-gradients can reduce the variance of our estimators. We call these value-gradient \emph{gradient-critics}.

\begin{definition}[Gradient-critic]
Consider two nodes $u$ and $v$, and a value-gradient $g_v$ for node $v$ with set $\cC$. If $\DxbyDy{v}{u}$ and $\DxbyDy{L^s}{v}$ are conditionally independent given $\cC$\footnote{See lemma \ref{lem:cond-ind-grads} in Appendix for a characterization of conditional independence between total derivatives.}, then we say the value-gradient is a \textbf{gradient-critic} for $v$ with respect to $u$.
\end{definition}

\begin{corollary}
If $\DxbyDy{v}{u}$ is deterministically computable from $\cC$, then $g_v(\cC)$ is a gradient-critic for $v$ with respect to $u$.
\end{corollary}
We can use gradient-critics in the backpropagation equation. First, we recall the equation for backpropagation and stochastic backpropagation. Let $u$ be an arbitrary node of $\cH$, and $\{v_1,\ldots, v_d\}$ be the children of $u$ in $\cG$. The backpropagation equations state that:
\begin{eqnarray}
    \label{eq:backprop}
    \DxbyDy{L^s}{u}=\sum_i \DxbyDy{L^s}{v_i} \dxbydy{v_i}{u}
\end{eqnarray}
From this we obtain the stochastic backpropagation equations:
\begin{align*}
    \E_{\cG}\left[\DxbyDy{L^s}{u}\right]=\E_{\cG}\left[\sum_i \DxbyDy{L^s}{v_i} \dxbydy{v_i}{u}\right]
\end{align*}

Gradient-critics allow for replacing these stochastic estimates by conditional expectations, potentially achieving lower variance:
\begin{theorem}\label{thm:grad-critic}
For each child $v_i$ of $v$, let $g_{v_i}$ be a gradient-critic for $v_i$ with respect to $u$. We then have:
\begin{align*}
    \E_{\cG}\left[\DxbyDy{L^s}{u}\right]=\E_{\cG}\left[\sum_i g_{v_i} \dxbydy{v_i}{u} \right].
\end{align*}
\end{theorem}

Note a similar intuition as the idea of critic defined in the previous section. In both cases, we want to evaluate the expectation of a product of two correlated random variables, and replace one by its expectation given a set which makes the variables conditionally independent.

\subsection{Horizon gradient-critic}

More generally, we do not have to limit ourselves to $\lbrace v_1, v_2, \ldots, v_d\rbrace$ being children of $u$. 
We define a \emph{separator set} for $u$ in $\cH$ to be a set $\lbrace v_1, v_2, \ldots, v_d \rbrace$ such that every deterministic path from $u$ to the loss $L^s$ is blocked by a $v_i\in\cH$. 
For simplicity, we further require the separator set to be \emph{unordered}, which means that for any $i \not = j$, $v_j$ cannot be an ancestor to $v_i$; we drop this assumption for a generalized result in the appendix~\ref{sec:extension}.
Under these assumptions, the backpropagation rule can be rewritten (see \citep{naumann2008optimal, TotalSCG}, also appendix~\ref{sec:extension}):
\begin{eqnarray}
    \label{eq:horizon-backprop}
    \E_{\cG}\left[\DxbyDy{L^s}{u}\right]=\E_{\cG}\left[\sum_i \DxbyDy{L^s}{v_i} \DxbyDy{v_i}{u}\right].
\end{eqnarray}

\begin{theorem}\label{thm:horizon-gradient}
Assume that for every $i$, $g_{v_i}$ is a gradient critic for $v_i$ with respect to $u$. We then have:
\begin{align*}
    \E_{\cG}\left[\DxbyDy{L^s}{u}\right]= \E_{\cG}\left[ \sum_i g_{v_i} \DxbyDy{v_i}{u}\right].
\end{align*}
\end{theorem}

This theorem allows us to `push' the horizon after which we start using gradient-critics.
It constitutes  the gradient equivalent of partial averaging, since it combines stochastic backpropagation (the terms $\DxbyDy{v_i}{u}$) and gradient critics $g_{v_i}$.

\subsection{The gradient-critic bootstrap}
\label{sec:gradientCritic:boostrapping}

We now show how the result from the previous section allows to derive a generic notion of bootstrapping for gradient-critics. 
\begin{theorem}[Gradient-critic bootstrap]\label{thm:gradientCritic:bootstrapping}
Consider a node $u$, unordered separator set $\lbrace v_1,\ldots, v_d\rbrace$. Consider value-gradient $g_u$ with set $\cC_u$ for node $u$, and $(g_{v_1}, g_{v_2}, \ldots, g_{v_d})$ with Markov sets $(\cC_{v_1},\ldots, \cC_{v_d})$ critics for $v_i$ with respect to $u$. Suppose that for all $i$, $\cC_{u}\subset \cC_{v_i}$. Then,
\begin{align}
    g_u =\sum_i \E_{\cC_{v_i}|\cC_u}\left[g_{v_i} \: \DxbyDy{v_i}{u} \right].
\end{align}
\end{theorem}

\begin{figure}
    \centering
    \includegraphics[width=1\textwidth]{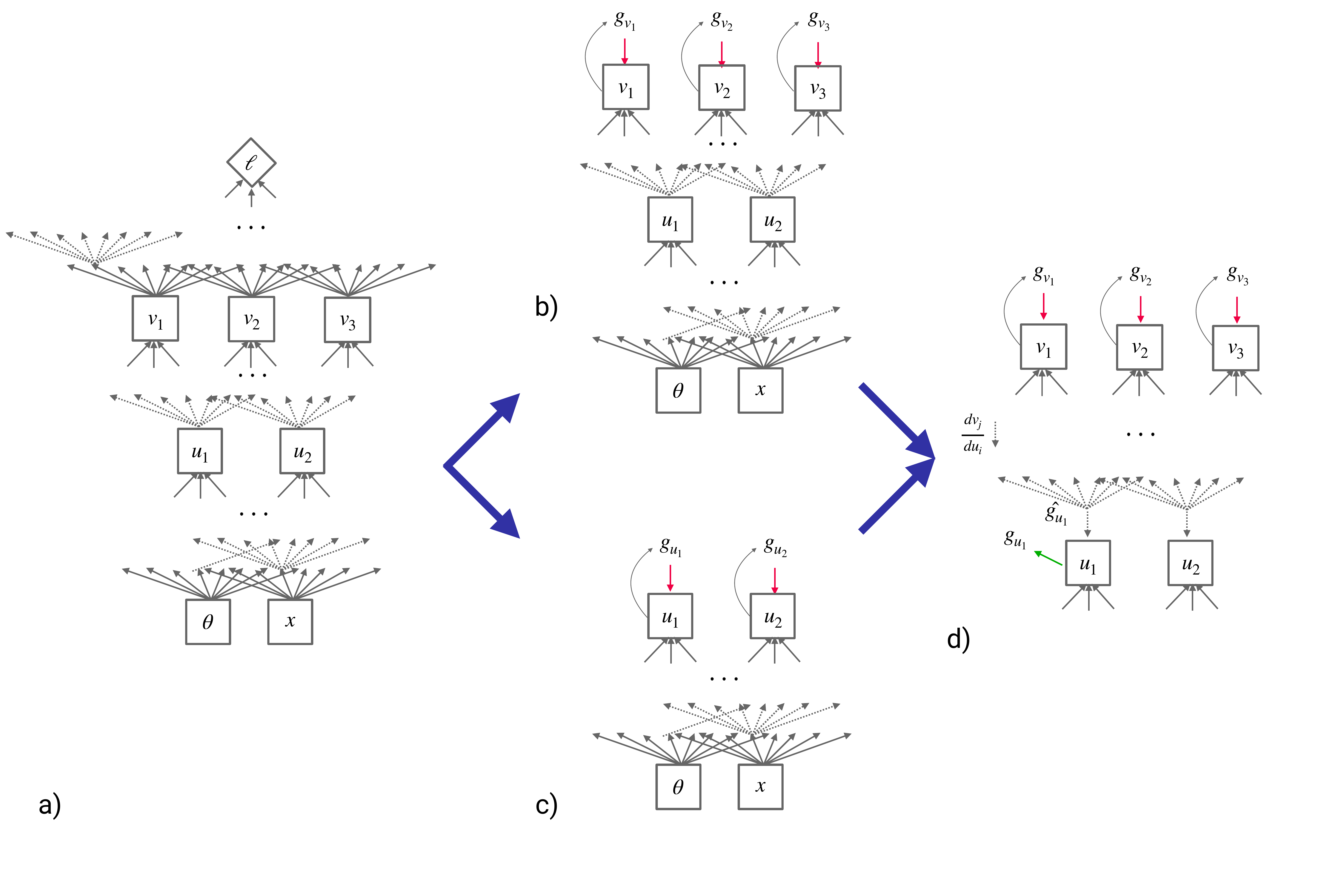}
    \caption{\emph{Gradient-critics.}\\
    a) A computation graph where $\lbrace u_1, u_2 \rbrace$ and $\lbrace v_1, v_2, v_3 \rbrace$ are both unordered separator sets. For the latter for instance, every path from $\theta$ to $\ell$ goes through $\lbrace v_1, v_2, v_3 \rbrace$, and given the set of all parents of $\lbrace v_1, v_2, v_3 \rbrace$, they can be computed in any order. \\
    b) The (horizon) gradient-critic for set $\lbrace v_1, v_2, v_3 \rbrace$:
    the forward and backward computations occurring after the separator set can be replaced by a gradient-critic:$\DxbyDy{\ell}{\theta}=\sum_i g_{v_i} \DxbyDy{v_i}{\theta}$.\\
    c) The gradient-critic for set $\lbrace u_1, u_2 \rbrace$.\\
    d) The gradient-critic bootstrap: for node $u_1$, one can use the separator set $\lbrace v_1, v_2, v_3 \rbrace$ to estimate the gradient of the loss with respect to $u_1$ as the `partial gradient-critic' $\hat g_{u_1}=\sum_i g_{v_i} \DxbyDy{v_i}{u_1}$. The gradient critic $g_{u_1}$ can be regressed against either the empirical gradient $\DxbyDy{L}{u_1}$ or the partially averaged gradient $\hat g_{u_1}$.
    \label{fig:horizon-grad}}
\end{figure}

\subsection{Gradient-critic and gradient of critic}
\label{sec:gradient_critic:gradient_of_critic}

The section above proposes an operational definition of a gradient critic, in that one can replace the sampled gradient $\DxbyDy{L^s}{u}$ by the expectation of the gradient $g_u$. A natural question follows -- is a value-gradient the gradient of a value function? Similarly, is a gradient-critic the gradient of a critic function? 

It is in general not true that the value-gradient must be the gradient of a value function. However, if the critic set is Markov, the gradient-critic is the gradient of the critic.
\begin{theorem}\label{thm:gradientCritic:criticGradient}
Consider a node $v$ and critic set $\cC$, and corresponding critic value function $Q(\cC)$ and gradient-critic $g_v(\cC)$. If $\cC$ is Markov for $v$, then we have:
\begin{align*}
    \DxbyDy{Q(\cC)}{v}=g_v.
\end{align*}
\end{theorem}
This characterization of the gradient-critic as gradient of a critic plays a key role in using reparametrization techniques when gradients are not computable. For instance, in a continuous control application of reinforcement learning, the state of the environment can be assumed to be an unknown but differentiable function of the previous state and of the action. In this context, a critic can readily be learned by predicting total costs. By the argument above, the gradient of this critic actually corresponds to the gradient-critic of the unknown environment dynamics. This technique is at the heart of differentiable policy gradients \citep{lillicrap2015continuous} and stochastic value gradients \citep{heess2015learning}.

When estimating the gradient critic from the critic, one needs to make sure that the conditional distribution on $\cC$ conditional on $\cG\setminus\cC$ has `full density' (i.e. that the loss function can be evaluated in a neighborhood of the values of $\cC$), otherwise the resulting gradient estimate will be incorrect. This is an issue for instance if variables in $\cC$ are deterministic function of one another. To address this issue, one can sample $\cC$ from a different distribution than $p(\cC|\cG_v\setminus \cC)$, for instance by injecting additional noise in the variables. One may have to use bootstrap equation instead of regression on return, since other we would be estimating the critic of a different graph (with added noise). See for instance \citep{silver2014deterministic, lillicrap2015continuous}.

\subsection{Gradient-critic approximation and computation}

Following the arguments regarding conditional expectation and square minimization from section~\ref{subsec:Values}, we know that $g_v$ satisfies the following minimization problem:
\begin{align*}
g_v(\cC)=& \text{argmin}_{g_{c_v}}\: \E_{\cG\setminus \cC|\cC}\left[\left(g_{c_v}-\DxbyDy{L^s}{v}\right)^2\right]
\end{align*}
For a parametric approximation $g^\phi_v$, and using the same weighting scheme as section~\ref{sec:comp-val-crit}, it follows that:
\begin{eqnarray}
\mathcal{L}^\phi&=& \E_{\cC}\left[\E_{\cG\setminus \cC|\cC}\left[\left(g^\phi_v(\cC)-\DxbyDy{L^s}{v}\right)^2\right]\right]\nonumber\\ 
&=& \E_\cG\left[\left(g^\phi_v(\cC)-\DxbyDy{L^s}{v}\right)^2\right] \nonumber\\
\DxbyDy {\mathcal{L}^\phi}{\phi} &= &\E_\cG\left[\DbyDp {g^\phi_v(\cC)} \left(g^\phi_v(\cC)-\DxbyDy{L^s}{v}\right)\right] \label{eq:ROR2}
\end{eqnarray}

Finally, if $\cC$ is Markovian for $v$, from Theorem~\ref{thm:gradientCritic:criticGradient}, the gradient-critic $g^\phi_v$ can be defined in two ways: first, as the critic of a gradient ($\E[\DxbyDy{L^s}{v}|\cC])$, and second, as the gradient of a critic ($\DxbyDy{}{v}\E[L|\cC]=\DxbyDy{}{v}Q(\cC)$).

In this case, it makes sense to parameterize $g^\phi_v$ as the derivative of a function $Q^\phi(\cC)$, where $v\in \cC$, i.e.\ define $g^\phi_v(\cC)=\DxbyDy{Q^\phi(\cC)}{v}$.
The gradient-critic can therefore defined directly by the gradient-critic loss $\E_\cG\left[\left(\DxbyDy{Q^\phi(\cC)}{v}-\DxbyDy{L^s}{v}\right)^2\right]$, and indirectly by the critic loss $\E_\cG\left[\big(Q^\phi(\cC)-L\big)^2\right]$\footnote{For the critic loss, note however that approximating a function well does not imply the corresponding gradients are close.}.
It may therefore makes sense to combine them:
\begin{eqnarray}
\label{eq:Sobolev}
\mathcal{L}^\phi = \E_\cG\Big[\alpha\big(Q^\phi(\cC)-L\big)^2+\beta \left(\DxbyDy{Q^\phi(\cC)}{v}-\DxbyDy{L^s}{v}\right)^2\Big] 
\end{eqnarray}
where $\alpha, \beta$ are relative weights for each norm. This is called a \emph{Sobolev} norm, see also \citep{czarnecki2017sobolev}.

\section{Applications in the RL literature}
\label{sec:Examples}
In this section and next we discuss multiple examples of models from the RL and probabilistic modeling literature. We present the corresponding SCGs and associate surrogate losses.

\subsection{Markov decision processes and partially observed Markov decision processes}
\label{sec:Appendix:Examples:MDP}
Figure \ref{fig:scg_examples:MDP} shows several examples of decision processes from the reinforcement learning literature. For simplicity we focus on the undiscounted, finite horizon case but note that  generalizations to the discounted, infinite horizon case are straightforward.\footnote{
In fact, the finite horizon requires particular care in that policies and value functions become time-indexed but we ignore this here in favor of a simple notation.
}

\paragraph{Markov decision processes (MDPs)} 
As explained in in Fig.~\ref{fig:example_models} the MDP objective is the (in our case undiscounted) return $\E_\tau [ \sum_{t=1}^T r(s_t, a_t) ]$ where trajectories $\tau=(\ldots, s_t, a_t, \ldots)$ are drawn from the distribution obtained from the composition of policy $\pi_\theta$ and system dynamics: $\tau \sim p(\tau, \theta) = p(s_0) \prod_t \pi(a_t|s_t, \theta) p(s_{t+1}|s_t, a_t)$.

Fig.\ \ref{fig:scg_examples:MDP}(a) shows a vanilla, undiscounted MDP with a policy $\pi_\theta(a | s)$ parameterized by $\theta$.
A large number of different estimators have been proposed for this model using a variety of different critics including Monte-Carlo returns, $k$-step returns, $\lambda$-weighted returns etc.:
\begin{align}
    Q_{a_t}^{MC} &= \sum_{t'\geq t} r(s_{t'},a_{t'}) \\
    Q_{a_t}^0 &= Q_{a_t}(s_t, a_t)\\
    Q_{a_t}^{k} &= \sum_{t'=t}^{t+k} r(s_{t'},a_{t'}) + V(s_{t+k+1}), \ \ \ \ \ \forall k>0 \\
    Q_{a_t}^{\lambda} &= (1-\lambda) \sum_{k=0}^\infty \lambda^k Q_{a_t}^{k}.    
\end{align}
The $\lambda$-weighted returns are an instance of a convex combination of a different set of critics, in this case of $k$-step returns critics. $K$-step returns are examples of partial averages. These critics can be used both to estimate the advantage for a policy update using the policy gradient theorem.  Furthermore, they can be valuable as bootstrap targets for learning value functions. In general, the critic used for estimating the advantage can be different from the one used to construct a target for the value function update.

Fig.\ \ref{fig:scg_examples:MDP}(c) shows an example of an MDP where independence between the components of a multi-dimensional action is assumed,
corresponding to a factorized policies with $\pi_{\theta}(a_t \vert s_t) = \prod_i \pi_\theta(a_t^i \vert s_t)$.\footnote{
This is the predominant case in practice, especially in continuous action spaces where policies are frequently chosen to be factorized Gaussian distributions.
}
This motivates the use of action-conditional baselines \cite[e.g.][]{wu2018variance} or marginalized critics. For instance, the action-conditional baseline for action dimension $i$, in state $s_t$ is given by $V_{a_t^i}(s_t,a_t^{-i}) = \mathbb{E}_{a_t^i\vert s_t, a_t^{-i}}\left[  Q(s_t,a_t) \right]$, where $a_t^{-i}=(a_t^1,\ldots, a_t^{i-1}, a_t^{i_1}, \ldots a_t^N)$.
This is a valid baseline value function according to our Def.\ \ref{def:Baseline} as the remaining $a_t^{-i}$ are non-descendants of the action $a_t^i$.

\paragraph{Partially Observable Markov Decision Processes (POMDPs)}
Fig.\ \ref{fig:scg_examples:MDP}(c,d,e) show two examples of POMDPs. (c) shows the standard setup where the state $s_t$ is unobserved. Information about the state of the environment are available to the agent only via observations $o_t \sim p(\cdot | s_t)$. Observations typically provide only partial information about the state. To act optimally (or to predict the value in some state) the agent therefore needs to infer (the distribution over) the state at timestep $t$ given the interaction history $h_t = (o_0, a_0, \dots, o_t)$: $p(s_t | h_t)$. The aggregated interaction history is often referred to the internal agent state or belief state  $b_t$. Internal state and action choice are usually trained from return (prediction - when training value functions - and maximization - when optimizing the policy). Note that when training the internal state from returns only, there is no guarantee that $b_t$ will correspond to a true `belief state', e.g. the sufficient statistics of the filtering distribution $p(s_t | o_0, a_0, \dots o_t)$; for a discussion of differences between internal and belief states, see for instance~\citep{igl2018deep, gregor2018temporal,moreno2018nb}.

The Markov structure of the model naturally suggests that value functions for time step $t$ should be conditioned on the entire observation history up to time-step $t$. Since the policy shown in the figure is dependent on the entire observation history, a critic has to be conditioned on the entire observation history (through $b_t$) too in order to satisfy Definition \ref{def:critic}. Furthermore, the Markov property of the model also requires conditioning of the value function on the entire observation history for bootstrapping to be valid, independent of the dependency structure chosen for the policy.

But the theorems presented in this paper suggest interesting, less explored alternatives, in particular when the state $s_t$ is available at training time (but not at testing time, so that the agent policy cannot depend on $s_t$). For instance, since $s_t$ is a non-descendendant of the action $a_t$, the baseline for action $a_t$ may be trained to depend on the full state, for instance by using a value function $V(s_t, b_t)$. This baseline is likely to be significantly more accurate since it has access to information which may be very predictive of the return. It can also be used to help train the internal state $b_t$ of the agent better, since $V(s_t, b_t)$ is a valid, lower variance bootstrap target for training the value function $V(b_t)$, which in turn will affect the representation $b_t$ learned by the agent. $s_t$ may also be used for critics, for instance by using Q functions which depend on both the environment and agent state: $Q(s_t, b_t, a_t)$.

The example in Fig.\ \ref{fig:scg_examples:MDP}(d) is a special case of the general POMDP. Shown is a multi-task MDP with shared transition dynamics but with reward function that depends on a goal $g$ (which varies across tasks). The Markov structure suggests conditioning both policy and value functions on the goal variable $g$ if observed (in which case the model is a MDP with $g$ being part of the state), or the entire interaction history when $g$ is unobserved. As for a general POMDP setting, conditioning on $g$ or on the state-history is optional for the policy but required for bootstrapping of the value function (of course performance will suffer when the policy does not have access to sufficient information). 

Figure \ref{fig:scg_examples:MDP}(e) shows a similar setup but with the transition dynamics dependent on an unobserved variable $d$ affecting the dynamics. The same arguments as for (d) and (b) apply. The option of conditioning the value function but not the policy on the system dynamics $d$ has been exploited e.g.\ in the sim-to-real work in \citep{peng2017sim}. The setup gives better baselines and allows bootstrapping of the value function, while the  policy learns to act robustly without knowledge of the true dynamics $d$.

\begin{figure}[H]
    \centering
    \includegraphics[width=0.9\textwidth]{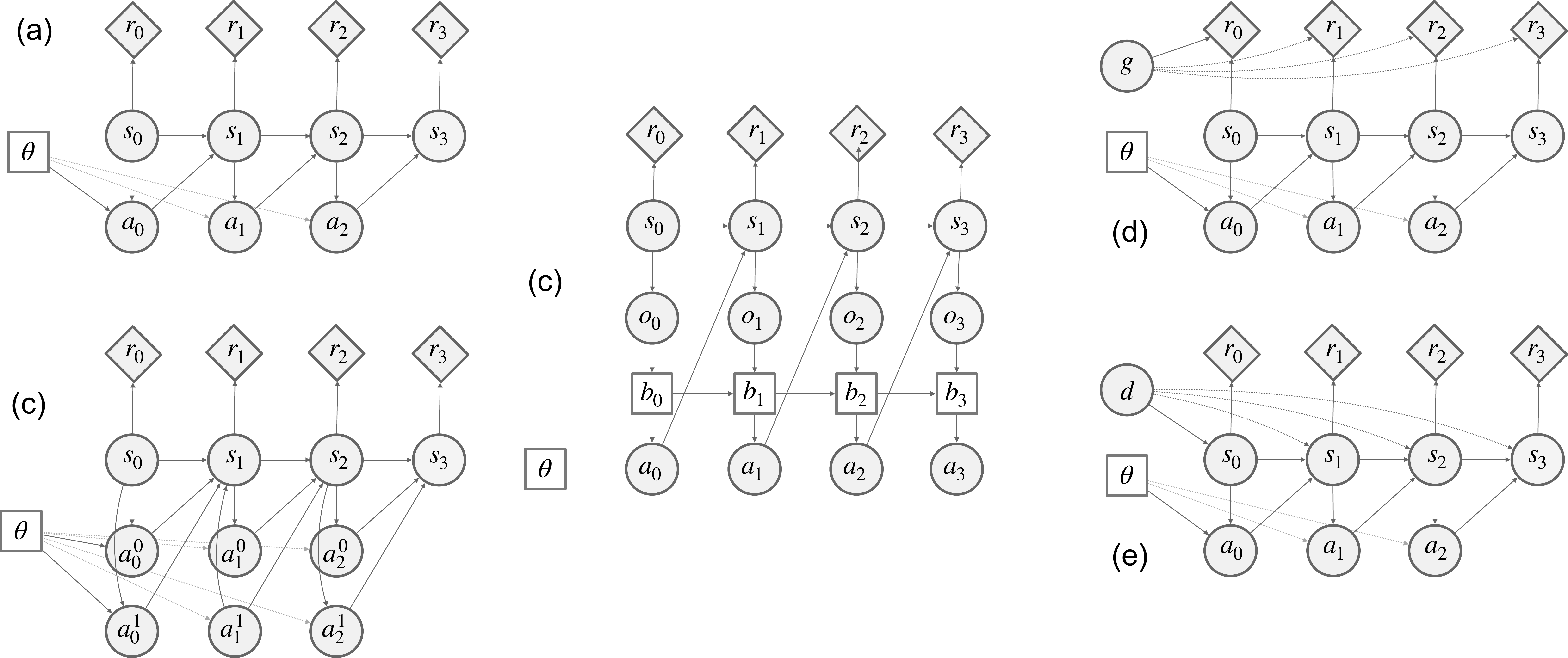}
    \caption{Example for SCG representations of decision processes. 
    (a) Vanilla MDP. 
    (b) MDP with two-dimensioal, independent  (factored) actions $a^0_t$ and $a^1_t$.
    (c) General POMDP.
    (d) Multi-task MDP with goal variable $g$.
    (e) MDP with unobserved system dynamics $d$.
    }
    
    \label{fig:scg_examples:MDP}
\end{figure}

\subsection{Reparameterized MDPs, value-gradients, and black-box policy search}

\paragraph{Reparameterization} The method of reparameterization is heavily exploited in the probabilistic modeling literature, but it can also be useful in RL by applying it to MDPs. Fig.\ \ref{fig:scg_examples:svg_black} (a) shows again the regular MDP from Fig.\ \ref{fig:scg_examples:MDP}a. Fig.\ \ref{fig:scg_examples:svg_black} (b) shows the fully reparameterized version of \cite{heess2015learning} where $\pi(a | s)$ and $p(s_t | s_{t-1}, a_{t-1})$ are replaced by deterministic functions of independent noise: $a_t = \pi(s_t, \epsilon_t)$ and $s_t = f(s_{t-1}, a_{t-1}, \xi_{t-1})$ respectively. (Deterministic policies and deterministic system dynamics can be treated as simple special cases of this general setup.)
If the (differentiable) system dynamics and distributions of the noise sources $\xi_t$ are know, we can use the standard backpropagation algorithms to compute policy gradients, as all stochastic nodes are root nodes in the graph (cf.\ e.g.\ Eqs.\ 6-8 in \citealt{heess2015learning}). 
When the system dynamics are unknown or non-differentiable, gradients of learned critics $Q(s,a)$ w.r.t.\ the actions can be used to obtain gradients both for deterministic and stochastic policies, e.g.
\begin{align}
    \frac{\partial J}{\partial \theta} = 
    \sum_t \mathbb{E} \left [ \frac{\partial Q_{a_t}}{\partial a_t} \frac{\partial \pi_\theta}{\partial \theta} \right ],
\end{align}
as discussed in Section \ref{sec:gradient_critic:gradient_of_critic}. For a deterministic policy this corresponds to the \emph{Deterministic Policy Gradients} (DPG) algorithm \citep{
silver2014deterministic,lillicrap2015continuous};  for stochastic policies it is a special case of the stochastic value gradients (SVG; \citealt{heess2015learning}) family, SVG(0).
The SVG(K) family also contains the analogue of partial averages. For instance, the policy gradient computed from 1-step MC returns (SVG(1)) is given by 
\begin{align}
    \frac{\partial J}{\partial \theta} = 
    \sum_t \mathbb{E} \left [ \left (\frac{\partial V_{s_{t+1}}}{\partial s_{t+1}} \frac{\partial s_{t+1}}{\partial a_t} + \frac{\partial r_t}{\partial a_t} \right ) \frac{\partial \pi_\theta}{\partial \theta} \right ].
\end{align}

We can further construct convex combinations (e.g.\ similar to $\lambda$-weighted returns) of such partial averages. These ideas have also been studied in the literature e.g.\ by \cite{werbos1982applications,fairbank2012value}. 

\paragraph{Black-box policy search} These methods ignore the temporal structure of the MDP and instead perform search directly at the level of the parameter $\theta$. 
A variety of different algorithms exist, with a particular simple form arising from representing an MDP in the equivalent from shown in Fig.\ \ref{fig:scg_examples:svg_black}(c).  The standard score-function estimator is used to learn a distribution over policy parameters $\theta$, which is parameterized by $\theta_0$, such as mean and standard deviation of $\theta$:
\begin{align}
    J(\theta_0) = \mathbb{E}_{\theta \vert \theta_0}\left[ 
    \mathbb{E}_{\cG\vert\theta} \left [ 
    \sum_t r(s_t, a_t)
    \right]
    \right]
\end{align}

The reparameterized version of this model is shown Fig.\ \ref{fig:scg_examples:svg_black}(d); it is closely related to a variety of recent proposals in the literature \cite{fortunato2017noisy, plappert2017parameter}.

\begin{figure}[H]
    \centering
    \includegraphics[width=0.9\textwidth]{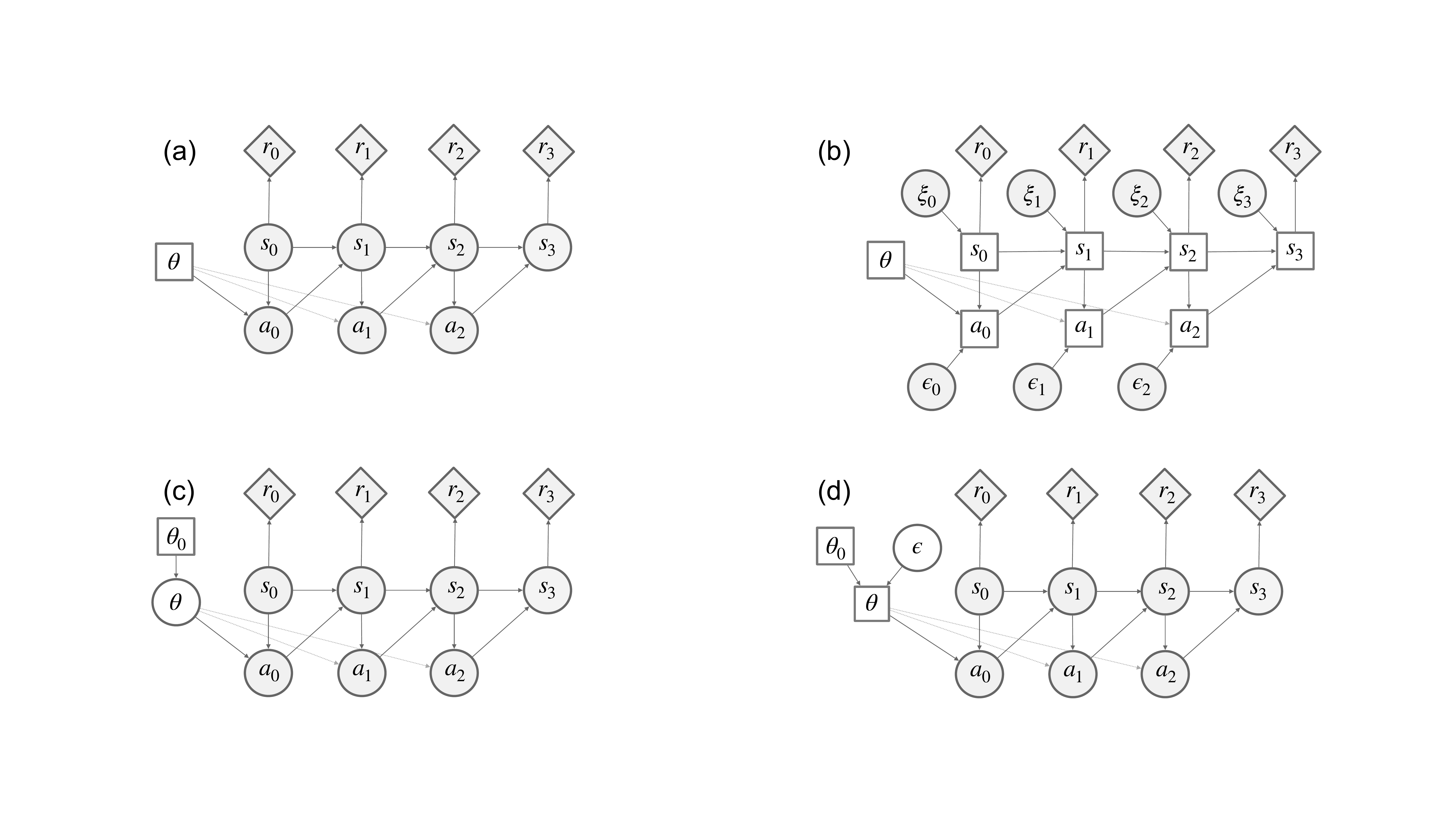}
    \caption{(a) MDP from Fig.\ \ref{fig:scg_examples:MDP} (for reference). (b) Fully reparameterized MDP as discussed in \cite{heess2015learning}. The stochastic action and state nodes have been replaced by deterministic nodes and independent noise variables. (c) Evolution strategies learn the parameters $\theta_0$ of a distribution over parameters $\theta$. (d) Same as (c) but with the distribution over $\theta$ reparameterized, similar to \emph{noisy networks}.}
    \label{fig:scg_examples:svg_black}
\end{figure}

\subsection{Hierarchical RL and Hierarchical policies}
\label{sec:Appendix:Examples:HRL}
\begin{figure}[H]
    \centering
    \includegraphics[width=0.9\textwidth]{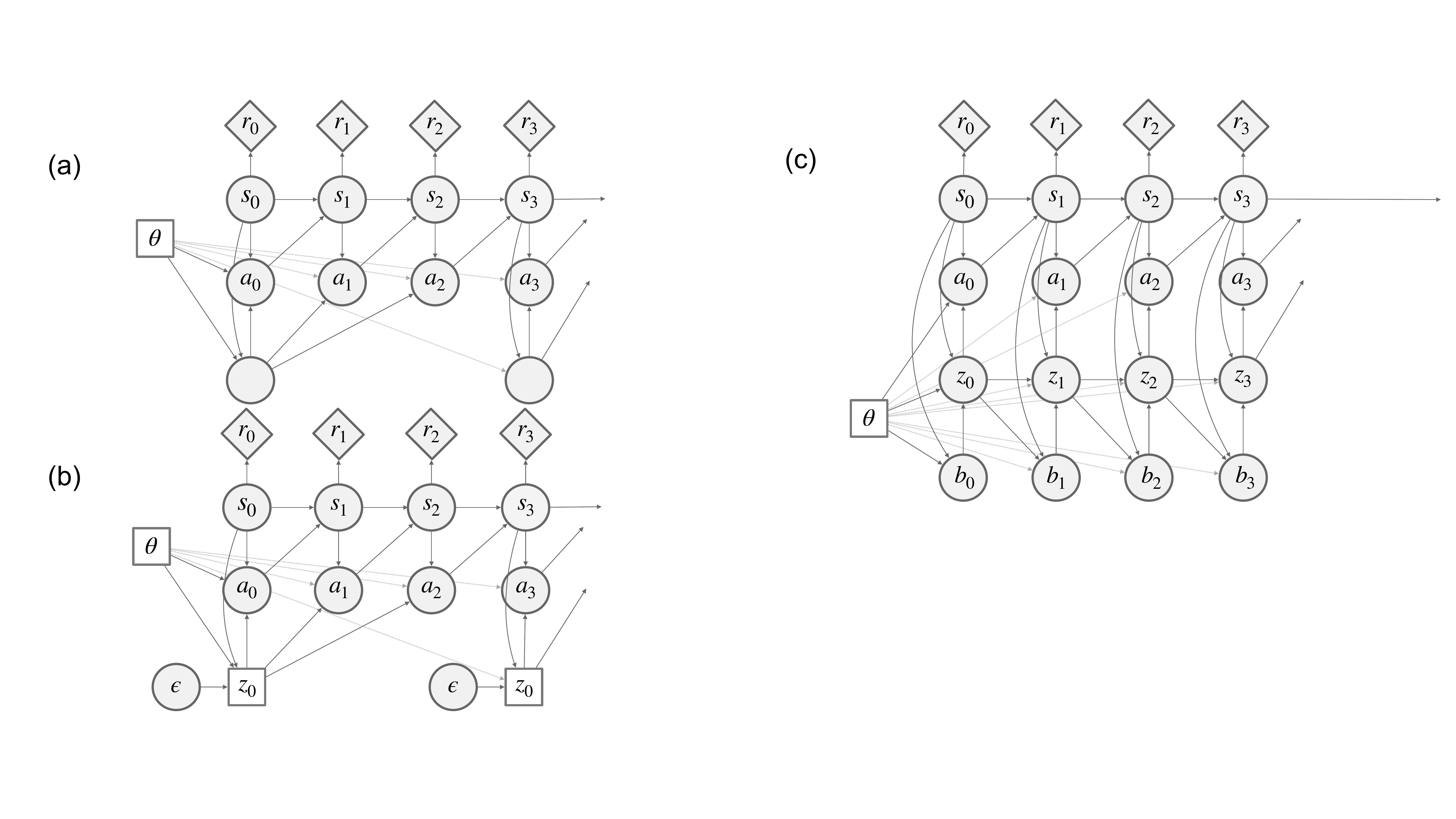}
    \caption{\emph{Examples of ``hierarchical'' policies: a) Policy with latent variable (``option'') that is fixed for $K=3$ steps. b) similar to (a) but the option is reparameterized (as e.g.\ in \cite{heess2016learning}). c) Options with variable duration as in the \emph{Option Critic} \citep{bacon2017option}.}}
    \label{fig:scg_examples:hrl}
\end{figure}

Figure \ref{fig:scg_examples:hrl} shows several simple examples of MDPs in which the policies have been augmented with latent variables. Such latent variables can, for instance, be seen as implementing the notion of options. For the discussion we assume that the objective remains unchanged (i.e.\ that no additional loss terms are introduced and we aim to optimize the full architecture to maximize expected reward).

Fig.~\ref{fig:scg_examples:hrl}(a) shows a simple example of a policy with ``options'' that have a fixed duration of three steps. For options of duration $M$ ($M=3$ in the example) the trajectory distribution is drawn from
$\tau \sim p(\tau; \theta) = p(s_0) \prod_{t=0,M,2M,\dots}\pi(z_t | s_t; \theta) \prod_{t'=0}^{M-1}\pi(a_{t+t'}|s_{t+t'}, z_t; \theta) p(s_{t+t'+1}|s_{t+t'}, a_{t+t'})$. Below we also use $[\cdot]_M$ to denote the ``current'' option, i.e. $[t]_M = \lfloor t/M \rfloor M$.

A variety of gradient estimators can be constructed. The main change compared to the choices discussed in \ref{sec:Appendix:Examples:MDP} is that critic and baselines have to reflect the dependence of $\tau_{\geq t}$ on $z_{[t]_M}$. 
Valid baselines for $z_t$ and $a_t$ would, for instance, be $V_{z_t}(s_t) = \E_{\tau_{\geq t} | s_t} [ \sum_{t' \geq t} r(s_{t'},a_{t'}) ]$ for $t=0,M, \dots$ and $V_{a_t}(s_t, z_{[t]_M}) = \E_{\tau_{\geq t} | s_t, z_{[t]_M}} [ \sum_{t' \geq t} r(s_{t'},a_{t'}) ]$, and similarly for critics:
\begin{align}
Q_{a_t}^0 &= Q_{a_t}(s_t, a_t, z_{[t]_M})= \E_{\tau_{\geq t} | s_t, a_t, z_{[t]_M}} [ \sum_{t' \geq t} r(s_{t'},a_{t'}) ]\\
Q_{a_t}^1 &= 
\begin{cases}
r(s_t,a_t) + V_{t+1}(s_{t+1},{z_{[t]_M}})& \textrm{if}~~[t+1]_M = [t]_M\\
r(s_t,a_t) + V_{t+1}(s_{t+1})&\mathrm{otherwise}\\
\end{cases}\\
Q_{z_t}^0 &= Q_{z_t}(s_t,z_t)= \E_{\tau_{\geq t} | s_t, z_t} [ \sum_{t' \geq t} r(s_{t'},a_{t'}) ] ~~~ \mathrm{for}~t=0,M,2M, \dots,
\end{align}
where we have made explicit that value functions may depend directly on the time step (due to the fixed duration of the option).

Fig.~\ref{fig:scg_examples:hrl}(b) shows the same model but with the random variables $z_t$ being replaced by a deterministic function of independent noise $\epsilon_t$. This allows the gradient with respect to $\theta$ at $z_t$ to be computed by backpropagation, e.g.\
\begin{align}
    \frac{\partial z_t}{\partial \theta} = 
    \sum_{t' = t}^{t+M-1} \E_{\tau_{\geq t} | s_t, z_t} \left [ \frac{\partial Q_{a_{t'}}}{\partial a_{t'}}
    \frac{\partial a_{t'}}{z_t}
    \frac{\partial z_t}{\partial \theta} \right ],
\end{align}
where $Q_{a_{t'}}$ may be a function of $z_t$ as discussed in the previous paragraph.

Fig.~\ref{fig:scg_examples:hrl}(c) shows a more complex which captures the essential features e.g.\ of the \emph{Option Critic} architecture of \cite{bacon2017option}. Unlike in (a,b) the option duration is variable and option termination depends on the state. This can be modeled with a binary random variable $b_t$ which controls whether $z$ remains unchanged compared to the previous timestep. The full trajectory distribution is given by
$\tau \sim p(\tau; \theta) = p(s_0)p(z_0|s_0; \theta)p(a_0 | s_0,z_0; \theta) \prod_{t > 0} p(s_{t+1} | s_t,a_t) \pi(b_t | s_t) \pi(z_t | s_t, b_t, z_{t-1}; \theta) \pi(a_t|s_t, z_t; \theta)$ where $p(z_t | s_t, b_t, z_{t-1}; \theta)= \delta(z_t - z_{t-1})^{b_t} \pi(z_t | s_t; \theta)^{b_t-1}$. 

Value functions of interest are, for instance,
$Q_{a_t}(s_t,a_t,z_t) = 
\E_{\tau_{\geq t}|s_t,z_t}[ \sum_t r(s_t,a_t)]$ (note the dependence on $z_t$ due to the dependence of future time steps on that value); 
$Q_{z_t}(s_t,z_t) = V_{s_t,z_t}(s_t,z_t) = \E_{a_t,s_{t+1}|s_t,z_t}[ r(s_t,a_t) + V_{s_{t+1},z_t}(s_{t+1},z_t) ] =  \E_{\tau_{\geq t} | z_t,s_t }[ \sum_{t'\geq t} r(s_t,a_t) ]] = \E_{a_t|s_t,z_t}[  Q_{a_t}(s_t,a_t,z_t)]$,
as well as
$V_{s_t}(s_t,z_{t-1}, b_t=1) = V_{s_t}(s_t) = \E_{z_t|s_t}[ \E_{\tau_{\geq t} | z_t,s_t }[ \sum_{t'\geq t} r(s_t,a_t) ]]$,
$V_{s_t,t_{t-1}}(s_t,z_{t-1}) = \pi(B_t = 1 | s_t) V(s_t) + (1-\pi(B_t = 1)) V_{s_t,z_t}(s_t, z_{t-1})$. Whereas the former two value functions are of interest as critics, the latter two are primarily interesting for bootstrapping purposes.

\subsection{Multi-agent MDPs}
Multi-agent MDPs can be seen as special cases of MDPs with a particular factorization. Two examples are shown in Fig.\ \ref{fig:scg_examples:MA} for the fully and partially observed case respectively. 
The factored structure of the MDP suggests particular choices for value function and critic, including conditioning the baseline on other agents' actions in the same vein as for action conditional baselines \citep{foerster2017counterfactual}. A valid bootstrap target requires access to the full state.
For a collection of agents $a\in \mathcal A=(A,B,\ldots)$, denote $u^a_t$ the action of agent $a$ at time $t$. The tuple of all actions is $u=(a^A, a^B, \ldots)$; the tuple of all actions except that of agent $a$ is $u^{-a}$.
The shared state is $s_t$. The value function $Q(s_t, u_t)$ is a valid critic for action $u^a_t$ (Def.~\ref{def:BCVF}).
Since all actions other than $u^a_t$ are non-descendents of $u^a_t$, they can be used to form an improved, valid baseline $V(s_t, u_t^{-a})$. From the Bellman principle (Lemma.~\ref{lem:Bellman}), we also have $V(s_t, u_t^{-a})=\E[Q(s_t, u_t)]=\sum \pi(u_t^a|s_t) Q(s_t, u_t)$

Note that in the fully observed case (all agents observe their own as well as the other agents' states), the above essentially takes the same form as the factored action model from Fig.\ \ref{fig:scg_examples:MDP}, due to identical structure in the graphical model.

An actor-critic formulation that employs state-action critic to derive action-value gradients based has been considered by \cite{lowe2017multi}. Note that the use of action-value gradients means that no baseline is needed.

Fig.\ \ref{fig:scg_examples:MDP}b shows the POMDP formulation of the same problem in which each agent has access to only a partial observation of the full system. While centralized baselines and critics may still be desirable, ``factored'' baselines and critics conditioned on the full history of observations accessible to each agents policy are also valid according to the definition of a critic (Def.\ \ref{def:critic}) and admit bootstrapping according to Theorem \ref{theorem:boostrap}.

\begin{figure}[H]
    \centering
    \includegraphics[width=0.9\textwidth]{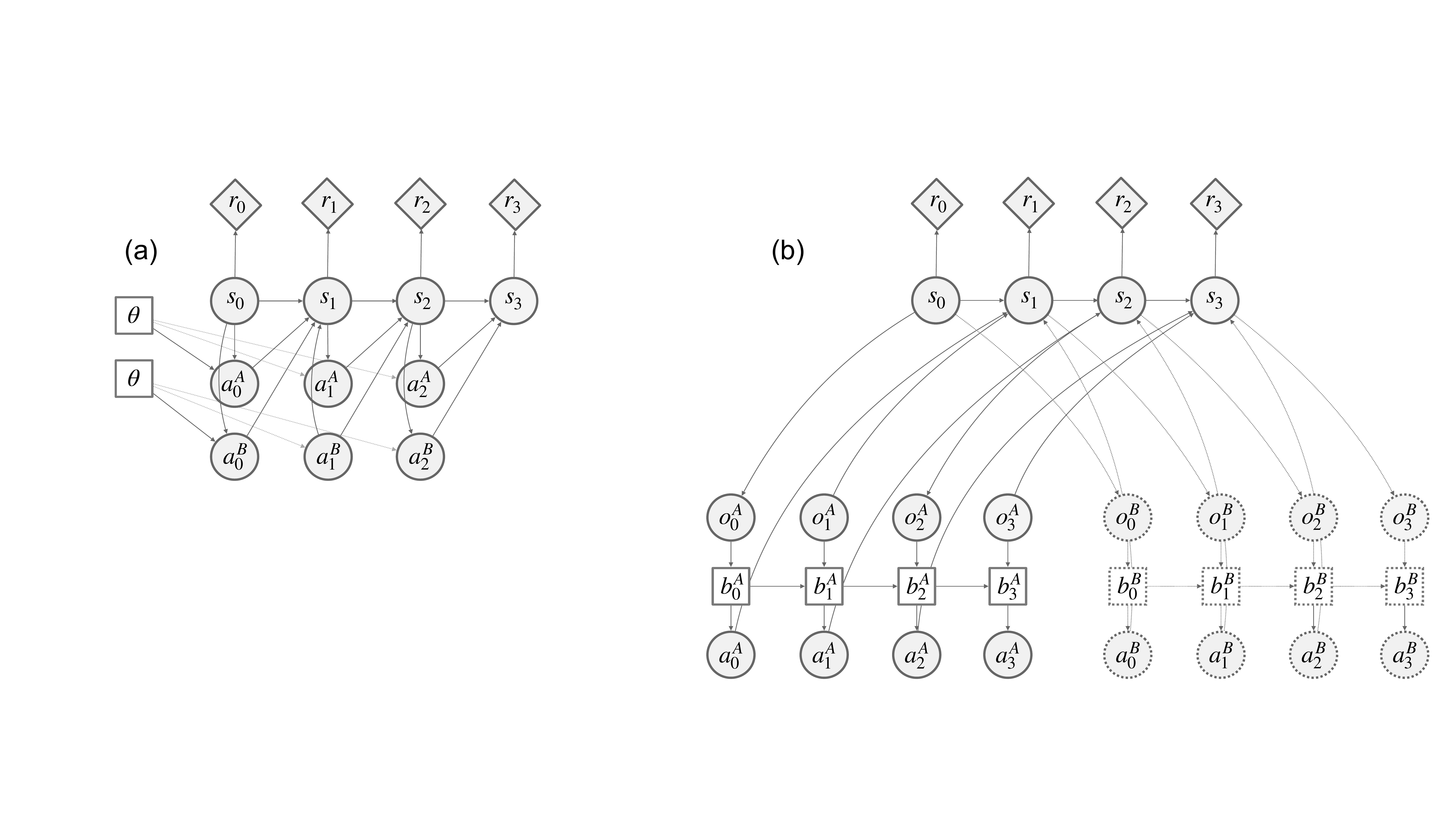}
    \caption{\emph{Examples: (a) Multiagent MDP. (b) Multiagent POMDP}}
    \label{fig:scg_examples:MA}
\end{figure}

\section{Applications in Probabilistic Modeling}

\subsection{Mapping inference in probabilistic models to stochastic computation graphs}

\label{sec:VIRL}

In this section, we briefly explicit a general technique to turn probabilistic inference problems into stochastic computation graphs, in particular in the framework of variational inference. This section closely follows \citep{weber2015reinforced}.

Let us consider an arbitrary latent variable model $p(\mb z,\mb x; \Theta)$, where $\mb x=(x_1,\ldots, x_M)$ represents the observations, $\mb z=(z_1,\ldots, z_K)$ the latent variables, and $\Theta=(\theta_1,\ldots, \theta_D)$ the parameters of $p$. 
Because $\mb x$, $\mb z$ and $\Theta$ play very similar roles in what follows, for simplicity of notation, let $y_j$ be a variable that represents a single latent $z_{j'}$ or observed variable $x_{j'}$ for some corresponding $j'$ (for instance we could have $y_j=z_j$ for $j=1,\ldots K$ and $y_{K+j}=x_j$ for $j=1,\ldots,M$)
Assume that the joint distribution of $(\mb z, \mb x)$ is that of a directed graphical model with directed acyclic graph $\cG^p$; for any variable $v \in \cG^p$, let $h^p_v$ be the parents of $v$ in $\cG^p$ (this includes parameters in $\Theta$). We then have:
$$p(\mb z, \mb x)=\prod_k p(z_k|h^p_{z_k}) \prod_{m} p(x_m|h^p_{x_{m}})=\prod_j p(y_j|h^p_{y_j}).$$
Let us consider a posterior distribution $q(z|x; \phi)$ which is also a directed graphical model $\cG^q$ ($\cG^p$ and $\cG^q$ may have different topology). For a node $v$ in $\cG^q$, let $h_v$ be the parents of $v$ in $\cG^q$ (again, this includes parameters in $\Theta$ or $\phi$). Then,
$$q(\mb z| \mb x)=\prod_k q(z_k|h_{z_k}).$$
The variational objective is $\E_{q}\left[\log p(x,z)-\log q(z|x)\right]$, which can be decomposed into a sum $\sum_j r_j(s_j)$, where for each $j$, we either have $r_j(s_j)=\log p(y_j|h^p_{y_j})$ with $s_j=(y_j, h_{y_j})$ or $r_j(s_j)=-\log q(z_j|h_{z_j})$ with $s_j=(z_j, h_{z_j})$.

The stochastic computation graph composed of $\cG^q$ and costs $r_j$ implements variational inference for $p$ using variational distribution $q$.

\subsection{Variational auto-encoders and Neural variational inference}

Simple stochastic computation graphs can be obtained from the combination of a latent variable model $p(z,x)$ and an amortized posterior network $q(z|x)$ \citep{gershman2014amortized}.
In the most common implementations, the latent variable model maps a latent variable $z$ with known prior (typically a normally distributed random vector, or a vector of Bernoulli random variables) to the distribution parameters (typically a Gaussian or Bernoulli) of the observation $x$. The model is in effect a mixture model (with an infinite number of components in the continuous case) where mixture components parameters are functions of the mixture `index'.

When the network is not reparametrized, the score function estimator is commonly known as NVIL (neural variational inference and learning) or automated/black-box VI and has been developed in various works \citep{paisley2012variational, wingate2013automated, ranganath2013black, mnih2014neural}. Baseline functions are commonly used.
When the distribution are reparametrizable, the model is often known as a VAE (variational autoencoder) \citep{kingma2013auto, rezende2014stochastic}; many variants of the VAE were since developed, often differing on the form of the posterior network, decoder, or variational bound.

Because of the simplicity and lack of structure of the model, it is not trivial to leverage critic or gradient-critics in these models. One could in principle learn a gradient-critic for NVIL and avoid the high-variance issue due to the score function estimator (see action parameter critics in section~\ref{sec:madness}); to our knowledge this has not been explicitly attempted. 
Multiple sample techniques \citep{mnih2016variational} can be used to lower variance of the estimators by estimating value function as the empirical average of independent samples (the validity of using other samples as baseline for one sample automatically follows from value functions from non-descendent sets of variables).

\begin{figure}[H]
    \centering
    \includegraphics[width=0.8\textwidth]{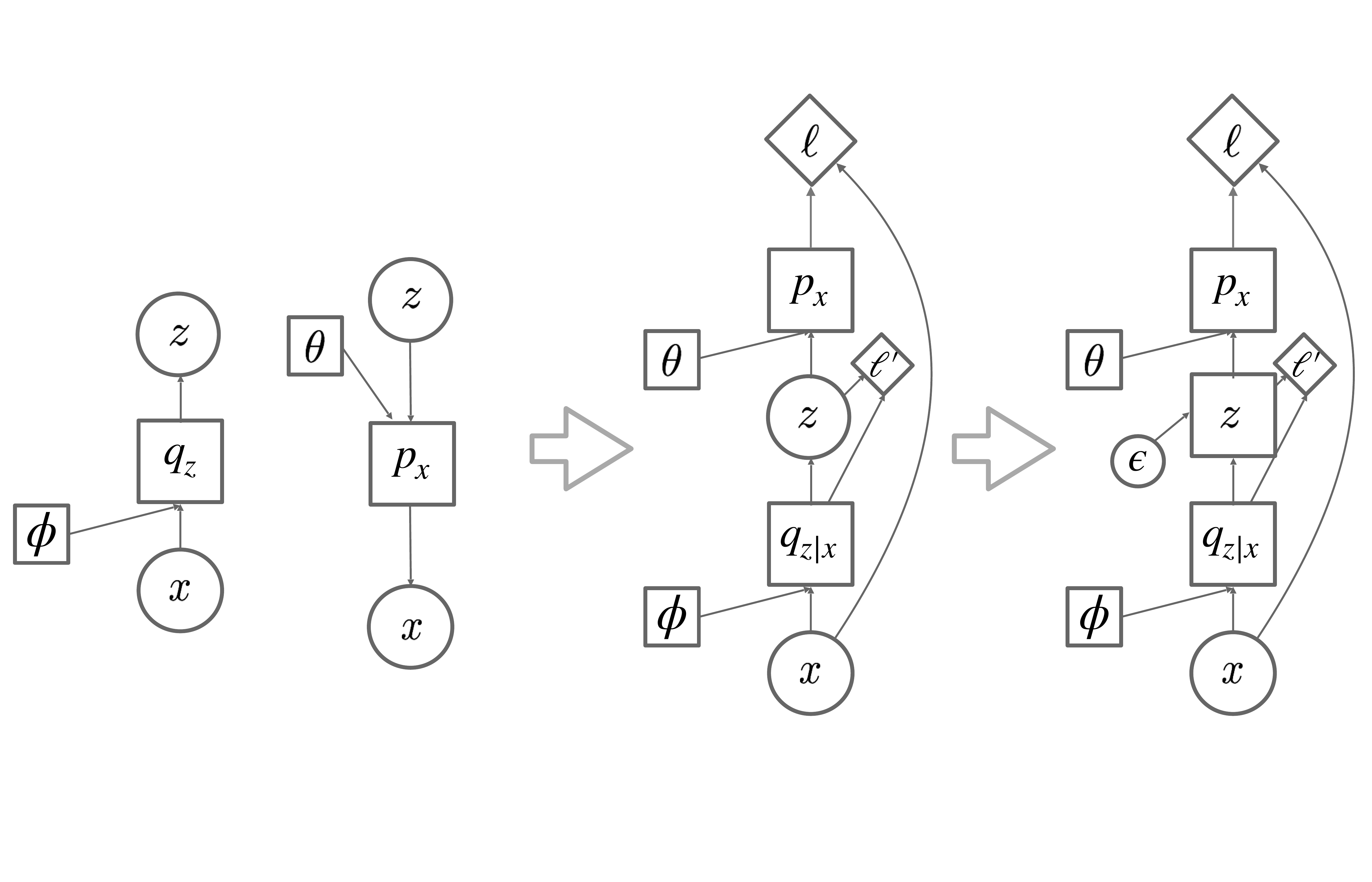}
    \caption{\emph{Variational Autoencoder and Neural Variational Inference.}\\
    a) Graphical model of neural latent variable model $p(x,z)$ and amortized posterior $q(z|x)$\\
    b) Computation graph for Neural Variational Inference\\
    c) Reparametrized computation graph for Variational Autoencoder.
    \label{fig:scg_examples:nvil-vae}}
\end{figure}

\subsection{State-space models}

\emph{State-space models} are powerful models of sequential data $\mb x=(x_1,\ldots, x_T)$, which capture dependencies between variables by way of a sequence of \emph{states} $z_t$, 

There is a large amount of recent work on these type of models, which differ in the precise details of model components \citep{bayer2014learning,chung2015recurrent,krishnan2015deep,archer2015black,fraccaro2016sequential,liu2017efficient,buesing2018learning}.

They generally consist of decoder or prior networks, which detail the generative process of states and observations,
and encoder or posterior networks, which estimate the distribution of latents given the observed data.  

Let $\mathbf{z}=(z_1,\ldots, z_T)$ be a state sequence and $\mathbf{x}=(x_1,\ldots, x_T)$ an observation sequence. We assume a general form of state-space model, where the joint state and observation likelihood can be written as $p(\mathbf{x}, \mathbf{z}) = \prod_t p(z_t\,|\,z_{t-1}) p(x_t\,|\,z_t)$.\footnote{For notational simplicity, $p(z_1\,|\,z_0)=p(z_1)$.} 
These models are commonly trained with a VAE-inspired bound, by computing a posterior $q(\mathbf{z}\,|\,\mathbf{x})$ over the states given the observations. Often, the posterior is decomposed autoregressively: $q(\mathbf{z}\,|\,\mathbf{x})=\prod_t q(z_t\,|\,z_{t-1}, \phi_t(\mathbf{x}))$, where $\phi_t$ is a function of $(x_1,\ldots, x_t)$ for filtering posteriors or the entire sequence $\mathbf{x}$ for smoothing posteriors. This leads to the following lower bound:
\begin{align}
\log p(\mathbf{x}) \geq \mathbb{E}_{\mathbf{z} \sim q(\mathbf{z}\,|\,\mathbf{x})} \left[{\sum}_t \log p(x_t\,|\,z_t) +\log p(z_t\,|\,z_{t-1})-\log q(z_t\,|\,z_{t-1}, \phi_t(\mathbf{x})) \right].
\label{eq:stardard_ssm_elbo}
\end{align}

Continuous state-space models are often reparametrized (see \citep{krishnan2015deep,buesing2018learning}). 
Comparatively less work investigates state-space model with discrete variables (see \citep{gan2015deep} for a recurrent sigmoid network, and \citep{eslami2016attend,kosiorek2018sequential} for hybrid models with both discrete and continuous variables). As with NVIL, discrete models suffer from higher variance estimates and could be combined with critics and gradient-critics, though it has not been extensively investigated.

\begin{figure}[H]
    \centering
    \includegraphics[width=0.9\textwidth]{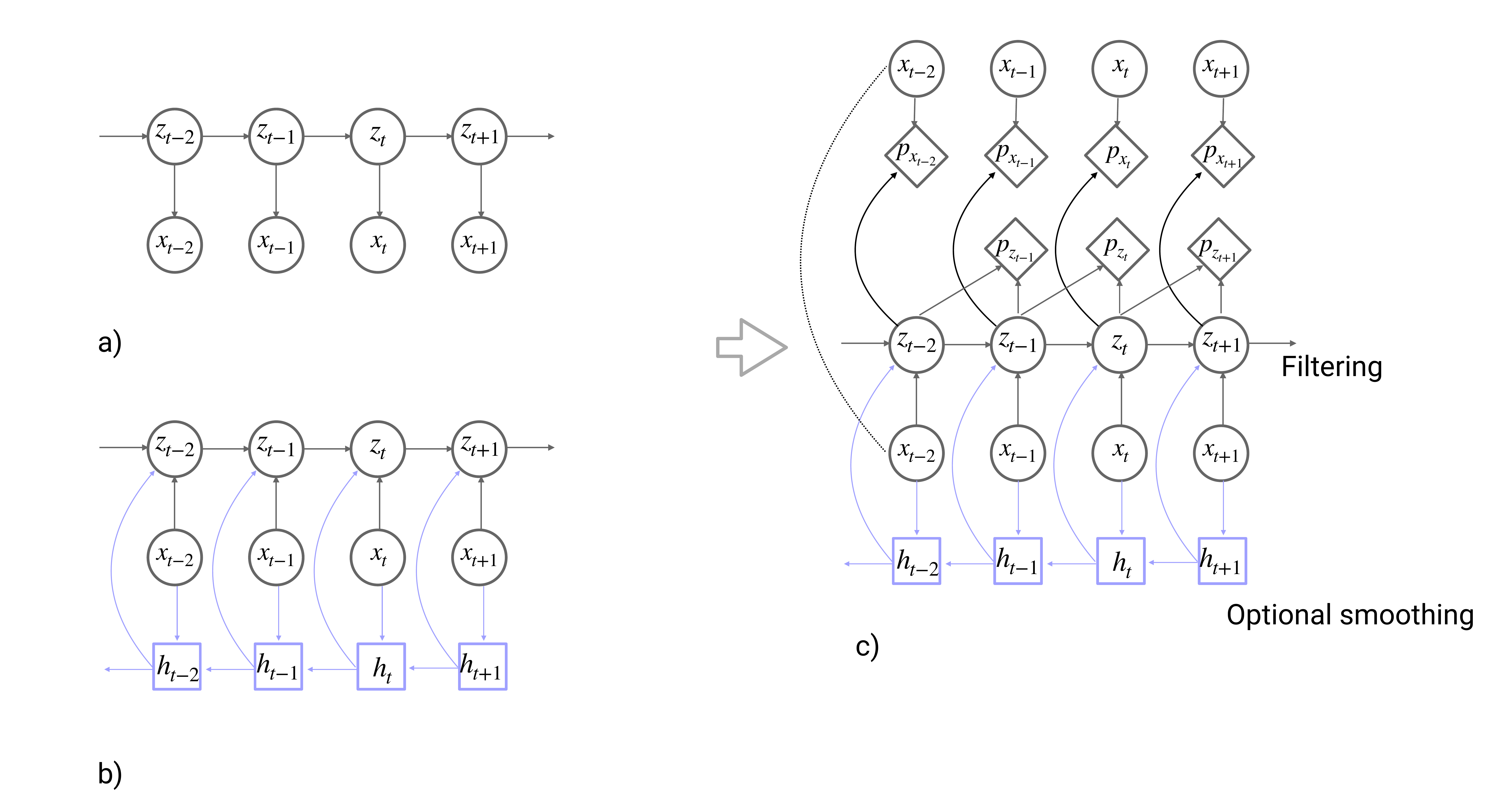}
    \caption{\emph{State-space model.}\\
    a) Forward model or decoder $p(z,zx)$\\
    b) Inverse model or encoder $q(z|x)$ based on filtering or smoothing (with purple arrows) posterior\\
    c) Stochastic computation graph
    \label{fig:scg_examples:ssm}}
\end{figure}

\subsubsection{More general inference schemes; RL as Inference}

The previous section addresses inference from a variational point of view, which maps closely to the RL problem\footnote{Note that an important distinction between RL and inference is that in RL, costs typically depend only on states and actions, while in inference, the cost depends on the parametrization of the random variables as well as on their samples; this is the reason why optimal posterior distributions are still entropic, as opposed to deterministic optimal policies in MDPs when not using entropy bonus.}. More general inference schemes can be used, from message-passing to sequential Monte-Carlo. While we do not extensively cover the connections, many of the notions developed in the previous sections can be extended to the general case. The main difference is that value-function, which conserve the intuition of `summarizing the future', are now defined in a `soft' fashion, by way of log-sum-exp operators (so do the corresponding bootstrap equations). For instance, for a state-space model, we can define a value function $V(z_t)$ as $V(z_t)=\log p(x_{t+1},\ldots, x_T|z_t)$ and $Q(z_{t}, z_{t+1})=\log p(x_{t+1},\ldots, x_T|z_t, z_{t+1});$ they are connected through a soft-Q update: $$V(z_t)=\log \E\left[ \exp(Q(z_{t}, z_{t+1}))|z_t\right].$$
The connection between inference and reinforcement learning is explored in details in \citep{levine2018reinforcement}, see also soft-actor critic methods \citep{haarnoja2018soft}. Just as value functions can be used to improve the quality of variational inference scheme (an idea proposed in \citep{weber2015reinforced}), soft-value functions can be used to improve the quality of other inference schemes, for instance sequential Monte-Carlo \citep{heng2017controlled, PlanningSMC}.

\section{Combination of estimators and critics}
\label{sec:Menu}

The techniques outlined above suggest a ``menu'' of choices for constructing gradient estimators for stochastic computation graphs. We lay out a few of these choices, highlighting how are results strictly generalize known methods from the literature. 

\paragraph{Reparameterization; use of score function and pathwise derivative estimators.}
Many distributions can be reparameterized, including discrete random variables \citep{maddison2016concrete,jang2016categorical}. This opens a choice between SF and PD estimator. The latter allows gradients to flow through the graph. Where exact gradients are not available (e.g.~in MDPs or in probabilistic programs and approximate Bayes computation \citep{meeds2014gps,ong2018variational}) gradients of critics can under certain conditions be used in combination with reparameterized distributions (see e.g.~\citep{heess2015learning}.

\paragraph{Grouping of random variables.}
For many graphs there is a natural grouping of random variables  that suggests obvious baseline and critic choices. Taking into account the detailed Markov structure of the computation graph may however reveal interesting alternatives. For instance, with an appropriate use of critics it could be beneficial to compute separate updates for each action dimension, marginalizing over the other action dimensions. Alternatively, independent action dimensions allow updates in which baselines are conditioned on the values chosen for other action dimensions. Such ideas have been exploited e.g.\ in work on action-dependent baselines \citep{wu2018variance}, and in multi-agent domains \citep{foerster2017counterfactual}. 
Other applications have been found in hierarchical RL, for instance in \citep{bacon2017option}, where the relation between options and actions directly informs the computation graph structure, in turn defining the correct value bootstrap equations and corresponding policy gradient theorem. The main results of these works can be rederived from the theorems outlined in this paper.

\paragraph{Use of value critics and baselines.}
The discussion in  \ref{sec:value:baseline_critic} highlights there are typically many different choices for constructing baselines and critics even beyond the choice of a particular variable grouping (e.g.\ the use of K-step returns or generalized advantage estimates in RL, \cite{schulman2015gradient}) even for simple graphs (such as chains). Which of the admissible estimators will be most appropriate in any given situation will be highly application specific. The use of critics for stochastic computation graphs was introduced in \citep{weber2015reinforced}, also independently explored in \citep{xu2018backprop}.

\paragraph{Use of gradient-critics.}
For reparameterized variables or general deterministic pathways through a computation graph gradient critics can be used. Gradient critics for a given node in the graph can be obtained by either directly approximating  the (expected) gradient of the downstream loss, or by approximating the value of the future loss terms (as for value critics) and then using the gradient of this approximation. Gradient-critics allow to conceptualize the links between related notions of value-gradient found in \citep{fairbank2012value, fairbank2014value}, stochastic value-gradients \citep{heess2015learning}, and synthetic gradients \citep{jaderberg2016decoupled,czarnecki2017sobolev}. 

\paragraph{Debiasing of estimators through policy-gradient correction.}
The use of critics or gradient-critics results in biased estimators. When using gradient-critic, it is often possible to debias the use of the gradient-critic by adding a correction term corresponding to the critic error. The resulting scheme is unbiased, but may or may not have lower variance than `naive' estimators which do not use critics at all. This is sometimes known as 'action-conditional' baselines in the literature \citep{tucker2018mirage}, and is also strongly related to Stein variational gradient \citep{liu2016stein}.
See App.~\ref{sec:debiasing} for more details.

\paragraph{Bootstrapping.} Targets for baselines and critics can be obtained in a variety of ways: For instance, they can be regressed directly onto empirical sums of downstream losses (``Monte Carlo'' returns in reinforcement learning). But targets can also constructed from other, downstream value or gradient approximations (e.g.\ ``K-step returns'' or ``$\lambda$-weighted returns'' in reinforcement learning), an idea discussed above under the name \emph{bootstrapping} (sections \ref{sec:valueCritic:boostrapping} and \ref{sec:gradientCritic:boostrapping}). The appropriate choice here will again be highly application specific.

\paragraph{Decoupled updates.} In its original form Theorem \ref{thm:main} requires a full and backward pass through the entire computation graph to compute a single sample approximation to its gradient. Through appropriate combination of surrogate signals and bootstrapping, however, updates for different parts of the graph can be decoupled to different extents. For instance, in reinforcement learning actor-critic algorithms compute updates to the policy parameters from single transitions $s_t, a_t, r_t, s_{t+1}$. The same ideas can be applied to general computation graphs where additional freedom (e.g.\ to set intermediate states; full access to all parts of the model) can allow even more flexible schemes (e.g.\ individual parts of the graph can be updated more frequently than others).

\section{A worked example}\label{sec:madness}
In this section, we go in more details through a simple chain graph example. This will allow us to present the menu of estimators discussed in the previous section in a concrete situation, without having to deal with the complexities arising from more structured graphs.

\begin{figure}[H]
    \centering
    \includegraphics[width=0.7\textwidth]{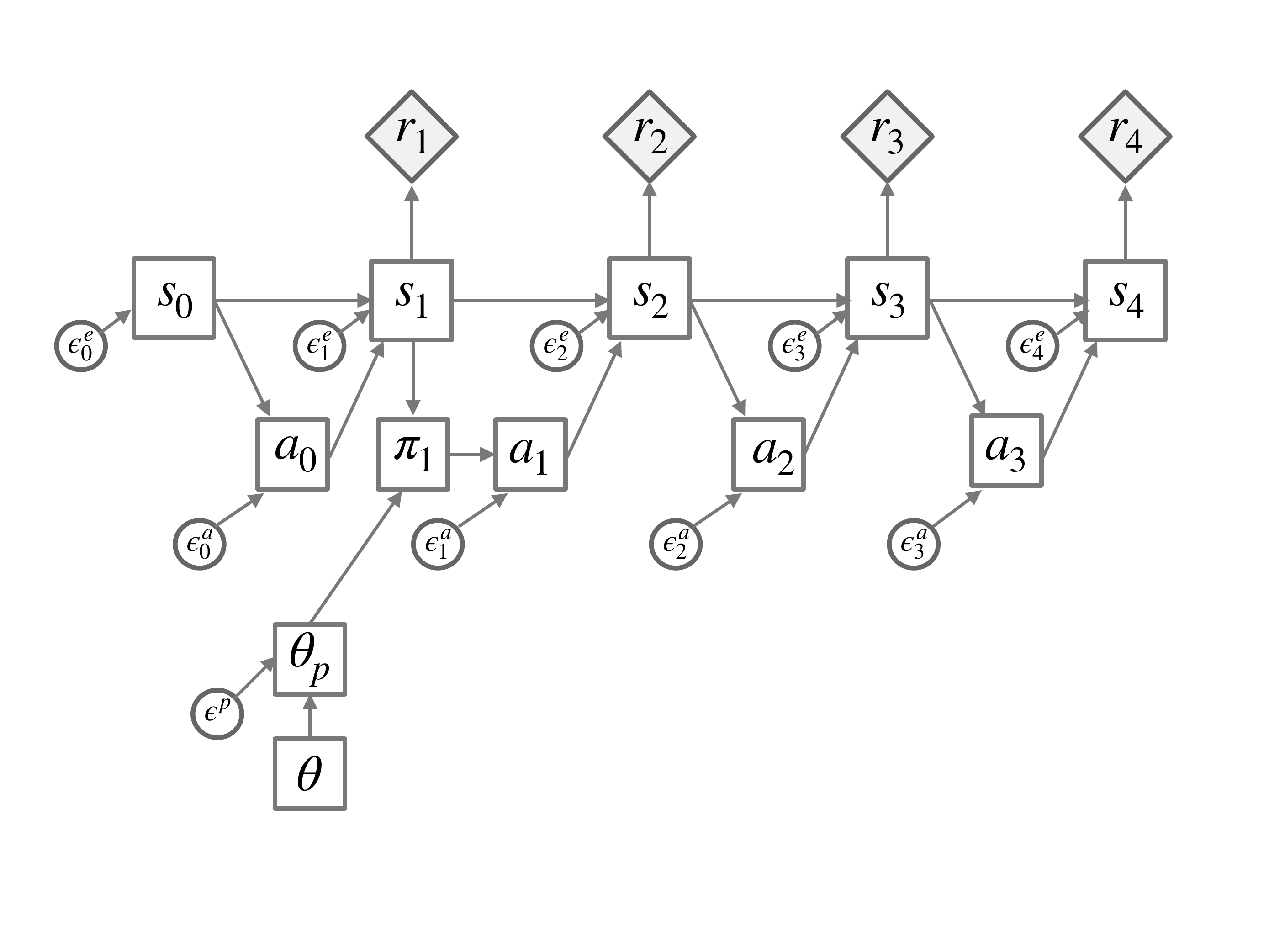}
    \caption{A simple MDP example, here shown fully reparametrized. Only parameters of the policy at time-step $t=1$ are shown.}
    \label{fig:worked_rl}
\end{figure}

We will use reinforcement learning terminology in order to connect to the wide literature on policy gradient and value-based methods, however the example is more general, and can be mapped to inference and learning in state-space models, time-series analysis, etc.

In this example, the state is fully observed by the agent, and we assume every function in the graph is known and differentiable.
We assume the graph can be arbitrarily reparametrized (with environment noise variables $\epsilon^e$, action noise $\epsilon^a$, parameter noise $\epsilon^p$). We assume a distribution over the parameters $\theta_p$ used to compute the policy; this distribution has hyperparameter $\theta$. That distribution may be a Dirac at $\theta$ (i.e.\ $\theta=\theta_p$).

The aim is to compute the parameters of the policy for the second action $a_1$ (this is done for simplicity but can easily be extended to the case where policy parameters are shared across time-steps); for reasons which will be apparent later, we make the parameters $\pi_1$ of the distribution of action $a_1$ a node of the graph (for other actions, those parameters are implicit in the state to action mappings); the mapping from $\pi_1$ and $\epsilon^a_1$ to $a_1$ is parameterless.
We detail different estimators of the gradient of the expected loss $\E[R]=\E[\sum_i r_i]$ with respect to $\theta$.

\paragraph{Black-box estimators: no reparametrization.\\}
The first and simplest estimator is the black-box or `evolution strategies' estimator, obtained by not reparametrizing $\theta_p$; it is given by $\DxbyDy{\log p(\theta_p|\theta)}{\theta} (R-V(\theta))$. It is hard to compute a value baseline or value critic for $R$ here. A critic that could be used in principle is an estimator $V(\theta_p)$ which averages over the entire environment and policy, and produces the expected return as a function of the sampled parameter $\theta^p$. The corresponding estimator would be $\DxbyDy{\log p(\theta_p|\theta)}{\theta} (V(\theta_p)-V(\theta))$. It is of course highly unlikely to be accurate in practice. 

\paragraph{Reparametrizing policy parameters: policy gradient algorithms.\\}
The next family of estimators is obtained by reparametrizing $\theta_p$, but not the action or the environment. If the conditional distribution of $\theta$ is a Dirac (i.e.\ noiseless), we have classical version of the estimators; if not, we have `noisy' or Bayesian estimators\citep{blundell2015weight,fortunato2017noisy, plappert2017parameter}.

We obtain `model-free' estimators, which differ by the choice of the critic.
One choice averages the entire future with a critic $Q(s_1,a_1)$, which leads to the update $$\dxbydy{\theta}{\theta_p}\DxbyDy{\log \pi_1(a_1)}{\theta}(Q(s_1,a_1)-V(s_1)).$$
Other choices involve the use of k-step returns, e.g. $$\dxbydy{\theta}{\theta_p}\DxbyDy{\log \pi_1(a_1)}{\theta}(r_1+V(s_2)-V(s_1)),$$ $$\dxbydy{\theta}{\theta_p}\DxbyDy{\log \pi_1(a_1)}{\theta}(r_1+r_2+V(s_3)-V(s_1)),$$ $$\ldots,$$
the use of TD$(\lambda)$ estimators, of empirical returns $$\dxbydy{\theta}{\theta_p}\DxbyDy{\log \pi_1(a_1)}{\theta}(R-V(s_1)).$$
These methods are all viable and used in practice. 

Another more unusual option is to use a parameter gradient-critic $g(\theta)$ which will 'average out' the entire environment and policy. Using a parameter gradient-critic leads to the update $\dxbydy{\theta}{\theta_p}g(\theta)$. The gradient critic $g(\theta)$ can be learned in two ways. First, by learning a critic $Q(\theta)$ which approximates the return $R$, and take the derivative with respect to $\theta$, i.\ e.\ $g(\theta)=\DxbyDy{Q(\theta)}{\theta}$. Second, by directly learning a gradient critic by regressing against a valid gradient update, for instance 
\begin{align*}
    g(\theta)=&\E\left[\DxbyDy{\log \pi_1(a_1)}{\theta}(R-V(s_1))\Big|\theta\right],\\
    \text{or}\:\:\:\:g(\theta)=&\E\left[\DxbyDy{\log \pi_1(a_1)}{\theta}(r_1+V(s_2)-V(s_1))\Big|\theta\right],
\end{align*} the latter of which learns a gradient-critic from a value critic. But no method which tries to approximates the entire actor-environment loop by a gradient $g(\theta)$ estimated solely from the parameter is likely to work in practice.

\paragraph{Reparametrizing actions.\\}
The next estimator involves reparametrizing the parameters, actions, but not the environment. In this situation, only one value critic is possible, the critic $Q(s_1,a_1)$ which averages the entire future. The corresponding estimator is the DDPG/SVG$(0)$ estimator, which is $$\dxbydy{\theta}{\theta_p}\dxbydy{\pi_1(a_1)}{\theta}\dxbydy{a_1}{\pi_1}\DxbyDy{Q(s_1,a_1)}{a_1}.$$
The term $\DxbyDy{Q(s_1,a_1)}{a_1}$ can be replaced by a gradient critic as we see next. 

\paragraph{Full reparametrization: differential dynamic programming and trajectory optimization.\\}
Finally, when reparametrizing the environment as well, we open the door to trajectory optimization type estimators. 
The general form is given by $$\dxbydy{\theta}{\theta_p}\dxbydy{\pi_1(a_1)}{\theta}\dxbydy{a_1}{\pi_1}g(s_1,a_1,\cdot),$$ where $g(s_1,a_1,\cdot)$ is some form of a gradient-critic for the future given $(s_1,a_1)$. The most classical involves using sampled environment gradients, and corresponds to the following expression, also known as SVG($\infty$):
$$\dxbydy{s_2}{a_1}(\dxbydy{r_2}{s_2}+\dxbydy{s_3}{s_2}(\dxbydy{r_3}{s_3}+\dxbydy{s_4}{s_3}\dxbydy{r_4}{s_4})).$$ 

But by using partially averaged gradient-critic, for instance a one-step gradient critic: $$\dxbydy{s_2}{a_1}(\dxbydy{r_2}{s_2}+\dxbydy{s_3}{s_2} g(s_3,a_3)),$$ where $g(s_3,a_3)$ either deterministically approximates the stochastic gradient $(\dxbydy{r_3}{s_3}+\dxbydy{s_4}{s_3}\dxbydy{r_4}{s_4}))$, or is the gradient of $Q(s_3,a_3)$. The most aggressive averaging would involve a gradient-critic $g(s_1,a_1)$ which averages the entire future, and which can be computed either by differentiating $Q(s_1,a_1)$, or by regression directly against a valid gradient target. 
\paragraph{Action parameter critics.\\}
Finally, let us consider a final, slightly more unusual gradient-critic estimator, explored in ~\citep{wierstra2007policy}, related to the DDP/SVG(0) operators, but where the critic is not conditioned on $(s_1,a_1)$ but on $(s_1,\pi_1)$ instead, resulting in the following: $$\dxbydy{\theta}{\theta_p}\dxbydy{\pi_1(a_1)}{\theta} g(s_1,\pi_1).$$
The gradient-critic $g(s_1,\pi_1)$ can be estimated in many ways. First, by regressing against a correct gradient estimate, for instance a quantity like $\dxbydy{a_1}{\pi_1}\DxbyDy{Q(s_1,a_1)}{a_1}$; this is only possible if $a_1$ can be differentiable reparametrized as a function of $\pi_1$. 

Second, as a gradient $\DxbyDy{V}{\pi_1}$ of the value $V(s_1, \pi_1)$ (note again this depends on the parameters $\pi_1$ and not the sample $a_1$); when using function approximators, this requires (during training at least) $\pi_1$ to be a stochastic function of $s_1$ (for the same reasons as highlighted in section~\ref{sec:gradient_critic:gradient_of_critic}).

Third, it can regressed against an estimate of the gradient where $a_1$ is not reparametrized, for instance $\DxbyDy{\log \pi_1(a_1)}{\pi_1}(R-V(s_1))$ or $\DxbyDy{\log \pi_1(a_1)}{\pi_1}(Q(s_1,a_1)-V(s_1))$. 

What is peculiar and particularly interesting about the second and third option is that they apply to non-differentiable actions, and that the gradient critic implicitly sums over all actions. This results in a potentially significantly lower variance gradient estimator for non-differentiable actions; note this does not require any relaxation to the discrete sampling. This could for instance be applied in inference and learning in discrete generative models, for instance as an alternative to the high variance estimators score function estimators.
\\
In this entire section, any bias introduced by the use of a critic or gradient-critic can be removed by using combined operators, as detailed in Appendix~\ref{sec:debiasing}.

Here, we don't discuss in depth the different options for bootstrap targets or linear combinations of critic, gradient critics and value functions, and yet, by making different choices on what quantities to condition on and which to average over, investigated a rich number of options available to us even in this very simple graph.

\section{Conclusion}

In this paper, we have provided a detailed discussion and mathematical analysis of credit assignment techniques for stochastic computation graphs. Our discussion explains and unifies existing algorithms, practices, and results obtained in a number of particular models and different fields of the ML literature. They also provide insights about the particular form of algorithms, highlighting how they naturally result from the constraints imposed by the computation graph structure, instead of ad-hoc solutions to particular problems.

The conceptual understanding and tools developed in this work do not just allow the derivation of existing solutions as special cases. Instead, they also highlight the fact that for any given model there typically is a menu of choices, each of which gives rise to a different gradient estimator with different advantages and disadvantages. For new models, these tools provide methodological guidance for the development of appropriate algorithms. In that sense our work emphasizes a similar separation of model and algorithm that has been proven fruitful in other domains, for instance in the probabilistic modeling and inference literature.

We believe that this separation as well as a good understanding of the underlying principles will become increasingly important as both models and training schemes become more complex and the distinction between different model classes blurs.

\section*{Acknowledgments}
We’d like to thank the many people for useful discussions and feedback on the research and the manuscript, including Ziyu Wang, S\'ebastien Racani\`ere, Yori Zwols, Chris Maddison, Arthur Guez and Andriy Mnih.

\clearpage
\newpage

\bibliographystyle{bibstyle}
\bibliography{bib}

\appendix
\newpage

\section{Gradient-critic extensions}

\label{sec:extension}

For variables $x$, $y$, and a set of variables $v_1, \ldots, v_d$, $\frac{dy}{dx}\dgiven_{(v_1,\ldots,v_d)}$  is the total gradient of $y$ with respect to $x$, but keeping the values of $v_i$ constant (i.\ e.\ the nodes $v_i$ are not back-propagated through when computing $\frac{dy}{dx}$).

\subsection{General form of the `horizon' backpropagation}

We first extend equation \eqref{eq:horizon-backprop} to the more general case where the separator set $(v_1, v_2, \ldots, v_d)$ is not necessarily unordered. We only require the nodes to be topologically ordered (i.e. if $v_i$ is an ancestor of $v_j$, $i\leq j$).
The more general form of `horizon' backpropagation is given by the following:
\begin{eqnarray}
    \DxbyDy{L^s}{v}=\sum_i \DxbyDy{L^s}{v_i} \DxbyDy{v_i}{v}\bigdgiven_{v_1,\ldots,v_{i-1}}
\end{eqnarray}

From which the corresponding horizon stochastic backpropagation immediately follows:
\begin{eqnarray}
    \label{eq:horizon-backprop-extended}
    \E_{\cG}\left[\DxbyDy{L^s}{v}\right]=\E_{\cG}\left[\sum_i \DxbyDy{L^s}{v_i} \DxbyDy{v_i}{v}\bigdgiven_{v_1,\ldots,v_{i-1}}\right]
\end{eqnarray}

\subsection{Notational hazards for gradients}
\label{sec:grad-notation}

Before going further, we work out through a few examples to highlight subtleties regarding the notation used for gradients, as well as the backpropagation equations given in equation~\eqref{eq:backprop} and in particular the version found in ~\eqref{eq:horizon-backprop-extended}. Note the latter is equivalent to backpropagation, but written in a way which tries to decouple the backpropagation updates which happen `upstream' and `downstream' of a group of nodes which are not necessarily the children of $v$.
\begin{figure}[H]
    \begin{minipage}[c]{0.5\textwidth}
        \includegraphics[width=1\textwidth]{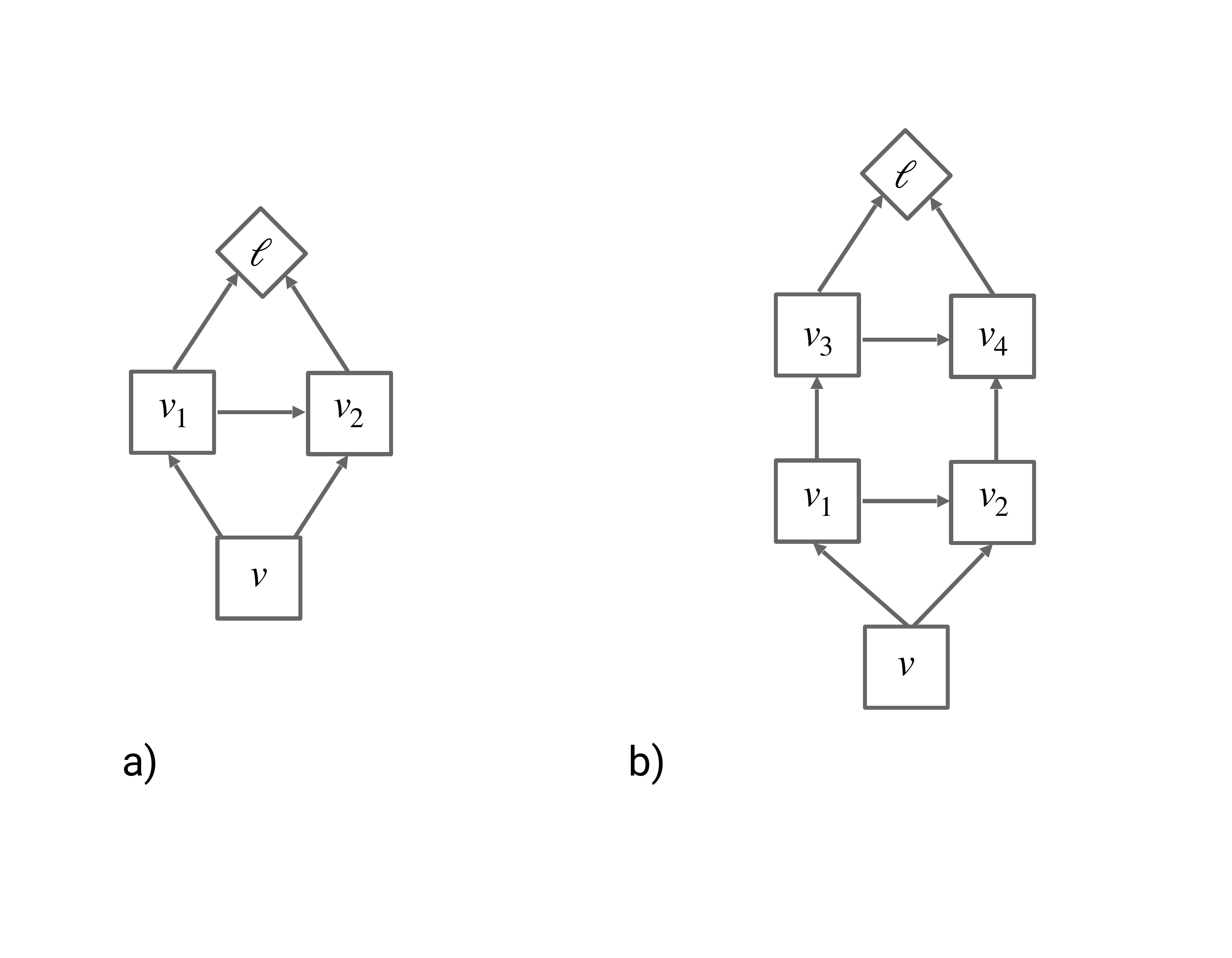}
    \end{minipage}\hfill
    \begin{minipage}[c]{0.45\textwidth}
        \captionof{figure}{
        \emph{Examples for backpropagation equations.}
        a) In this example, we have $v_1=f_1(v)$, $v_2=f_2(v,v_1)$ and $\ell=f_\ell(v_1,v_2)$.\\
        b) An example for the total gradient version of backpropagation. In this example, we have the following equations:\\
        At the first layer, $v_1=f_1(x)$ and $v_2=f_2(v_1,x)$\\
        At the second layer, $v_3=f_3(v_1)$ and $v_4=f_4(v_2,v_3)$\\
        Finally, $\ell=f_\ell(v_3,v_4)$.}
        \label{fig:gradients}
\end{minipage}
\end{figure}
We wish to compute $\DxbyDy{\ell}{v}$.
We first work through the first example. Following the chain rule, we have
\begin{eqnarray}
\DxbyDy{\ell}{v}=\dxbydy{\ell}{v_1}\dxbydy{v_1}{v}+\dxbydy{\ell}{v_2}\left(\dxbydy{v_2}{v_1}\dxbydy{v_1}{v}+\dxbydy{v_2}{v}\right)
\end{eqnarray}
We also have the following equations:
\begin{equation}
    \DxbyDy{\ell}{v_1}=\dxbydy{\ell}{v_1}+\dxbydy{\ell}{v_2}\dxbydy{v_2}{v_1} \hspace{0.5cm};\hspace{0.5cm} \DxbyDy{\ell}{v_2}=\dxbydy{\ell}{v_2} \hspace{0.5cm};\hspace{0.5cm}
    \DxbyDy{v_1}{v}=\dxbydy{v_1}{v} \hspace{0.5cm};\hspace{0.5cm}
    \DxbyDy{v_2}{v}=\dxbydy{v_2}{v}+\dxbydy{v_2}{v_1}\dxbydy{v_1}{v}
\end{equation}
From which we have the two equivalent forms of the chain rule. The first is equation~\eqref{eq:backprop}, and is the last computation performed by backpropagation: 
\begin{eqnarray}
\DxbyDy{\ell}{v}=\DxbyDy{\ell}{v_1}\dxbydy{v_1}{v}+\DxbyDy{\ell}{v_2}\dxbydy{v_2}{v}    
\end{eqnarray}
The second is related to forward differentiation:
\begin{eqnarray}
\label{eq:fwd-diff}
    \DxbyDy{\ell}{v}=\dxbydy{\ell}{v_1}\DxbyDy{v_1}{v}+\dxbydy{\ell}{v_2}\DxbyDy{v_2}{v}
\end{eqnarray}
We now move on to the second example. Again, fully expanding the chain rule, and summing over all paths, we have:
\begin{eqnarray}
\DxbyDy{\ell}{v}=\dxbydy{\ell}{v_3}\dxbydy{v_3}{v_1}\dxbydy{v_1}{v}+\dxbydy{\ell}{v_4}\dxbydy{v_4}{v_3}\dxbydy{v_3}{v_1}\dxbydy{v_1}{v}+\dxbydy{\ell}{v_4}\dxbydy{v_4}{v_2}\dxbydy{v_2}{v_1}\dxbydy{v_1}{v}+\dxbydy{\ell}{v_4}\dxbydy{v_4}{v_2}\dxbydy{v_2}{v}
\end{eqnarray}

Let us first consider the horizon set $\{v_2,v_3\}$. Note this set is unordered, we could compute $v_2$ first, then $v_3$, or vice-versa.  Equation~\eqref{eq:horizon-backprop} applies, which takes the form: $\DxbyDy{\ell}{v}=\DxbyDy{\ell}{v_3}\DxbyDy{v_3}{v}+\DxbyDy{\ell}{v_2}\DxbyDy{v_2}{v}$ follows, which can be verified from the following equations:
\begin{equation}
    \DxbyDy{\ell}{v_3}=\dxbydy{\ell}{v_3}+\dxbydy{\ell}{v_4}\dxbydy{v_4}{v_3}
    \hspace{0.25cm};\hspace{0.25cm}
    \DxbyDy{\ell}{v_2}=\dxbydy{\ell}{v_4}\dxbydy{v_4}{v_2} \hspace{0.25cm};\hspace{0.25cm}
    \DxbyDy{v_3}{v}=\dxbydy{v_3}{v_1}\dxbydy{v_1}{v} \hspace{0.25cm};\hspace{0.25cm}
    \DxbyDy{v_2}{v}=\dxbydy{v_2}{v_1}\dxbydy{v_1}{v}+\dxbydy{v_2}{v}
\end{equation}

Next, consider the horizon set $\{v_3,v_4\}$. Note that $v_3$ has to be computed before $v_4$.
First note we have $\DxbyDy{\ell}{v_3}=\dxbydy{\ell}{v_3}+\dxbydy{\ell}{v_4}\dxbydy{v_4}{v_3}$ and $\DxbyDy{\ell}{v_4}=\dxbydy{\ell}{v_4}$.
We can compute the total gradients of $v_3$ and $v_4$ with respect to $v$:
\begin{eqnarray}
\DxbyDy{v_3}{v}&=&\dxbydy{v_3}{v_1}\dxbydy{v_1}{v}\\
\DxbyDy{v_4}{v}&=&\dxbydy{v_4}{v_3}\dxbydy{v_3}{v_1}\dxbydy{v_1}{v}+\dxbydy{v_4}{v_2}\dxbydy{v_2}{v_1}\dxbydy{v_1}{v}+\dxbydy{v_4}{v_2}\dxbydy{v_2}{v}
\end{eqnarray}
A naive application of horizon backprop (equation~\eqref{eq:horizon-backprop}), i.e. $\DxbyDy{\ell}{v}=\DxbyDy{\ell}{v_3}\DxbyDy{v_3}{v}+\DxbyDy{\ell}{v_4}\DxbyDy{v_4}{v}$ is \textbf{incorrect}, because the term $\dxbydy{\ell}{v_4}\dxbydy{v_4}{v_3}\dxbydy{v_3}{v_1}\dxbydy{v_1}{v}$ is double-counted.
Following equation~\eqref{eq:horizon-backprop-extended} instead, we keep $v_3$ constant when computing the gradient of $v_4$ with respect to $v$, and obtain:
\begin{eqnarray}
\DxbyDy{v_4}{v}\bigdgiven_ {v_3}=\dxbydy{v_4}{v_2}\dxbydy{v_2}{v_1}\dxbydy{v_1}{v}+\dxbydy{v_4}{v_2}\dxbydy{v_2}{v}
\end{eqnarray}
and we correctly have $\DxbyDy{\ell}{v}=\DxbyDy{\ell}{v_3}\DxbyDy{v_3}{v}+\DxbyDy{\ell}{v_4}\DxbyDy{v_4}{v}\dgiven_{v_3}$, as per equation~\eqref{eq:horizon-backprop}. Note we also have $\DxbyDy{\ell}{v}=\DxbyDy{\ell}{v_3}\dgiven_{v_4}\DxbyDy{v_3}{v}+\DxbyDy{\ell}{v_4}\DxbyDy{v_4}{v}$, which is closer to equation~\eqref{eq:fwd-diff}.

\subsection{Extension to the gradient-critic theorem in the general case}

In this section, we provide an extension to the gradient-critic theorem \ref{thm:horizon-gradient} to the more general case.

\begin{theorem}\label{thm:horizon-gradient-extended}
If for each $i$, $\DxbyDy{v_i}{v}\bigdgiven_{v_1,\ldots,v_{i-1}}$ and $\DxbyDy{L^s}{v_i}$ are conditionally independent given $\cC_{v_i}$, 
\begin{align*}
    \E_{\cG}\left[\DxbyDy{L^s}{v}\right]= \E_{\cG}\left[ \sum_i g_{v_i} \DxbyDy{v_i}{v}\bigdgiven_{v_1,\ldots,v_{i-1}}\right]
\end{align*}
\end{theorem}
Similarly:
\begin{theorem}
Consider a node $v$, separator set $(v_1,\ldots, v_d)$. Consider critics $g_v$ with set $\cC_v$ and $(g_{v_1}, g_{v_2}, \ldots, g_{v_d})$ with sets $(\cC_{v_1},\ldots, \cC_{v_d})$. Suppose that for all $i$, $\cC_{v_i}\supset \cC_v$ and let $\Delta_i=\cC_{v_i}\setminus \cC_v$. Then,
\begin{align}
    g_v =\sum_i \E_{\Delta_i|\cC_v}\left[g_{v_i} \: \DxbyDy{v_i}{v}\bigdgiven_{{(v_{j})}_{j \leq i}} \right]      
\end{align}
\end{theorem}

Note that while the quantities $\DxbyDy{v_i}{v}\dgiven_{v_1,\ldots,v_{i-1}}$ may appear intimidating, they are implicitly the coefficients applied by backpropagation to the loss sensitivity $\DxbyDy{L^s}{v_i}$ at a node $v_i$ when computing $\DxbyDy{L^s}{v}$. In other words, if one replaces or initializes $\DxbyDy{L^s}{v_i}$ by $g_{v_i}$ in the backward graph and proceeds with backpropagation as usual, the resulting quantity compute at the root node will be effectively $\sum g_{v_i} \: \DxbyDy{v_i}{v}\dgiven_{{(v_{j})}_{j \leq i}}$; this implies that the bias in the gradient estimator will only come from the bias in critics themselves (not how they are used), and that for asymptotically perfect critics, the gradient estimator will be unbiased.

\section{Discussions}

\subsection{Optimal baseline and value function baseline}
\label{sec:opt-baseline}

The estimator of theorem \ref{thm:value-critic} is unbiased regardless of the choice of the baseline $B(\cB)$, which only affects the variance. 
A well-chosen baseline will reduce the variance of the gradient estimator. 
It is in general difficult to estimate or minimize the variance of the full estimator directly, however, we can derive a baseline which minimizes the variance of each $q_v$.

\begin{definition}
We say that a baseline set $\cB$ is \emph{congruent} with a critic set $\cC$ if it subset of $\cC$.
\end{definition}

\begin{theorem}\label{thm:optimal-baseline}
Let $s(v,\theta)$ be the score function $ \DbyDt \log p_v$, and consider a fixed critic set $\cC$ and congruent baseline set $\cB$. The following $B^\ast(\cB)$ is the baseline for $\cB$ which minimizes the variance of $G_v$:
\begin{align*}
    B^*(\cB) = \frac{\E_{\cG\setminus \cB|\cB}\left[s(v,\theta)^2 Q(\cC) \right]}{\E_{\cG\setminus \cB|\cB}\left[s(v,\theta)^2\right]}= \frac{\E_{\cG\setminus \cB|\cB}\left[s(v,\theta)^2 L(v) \right]}{\E_{\cG\setminus \cB|\cB}\left[s(v,\theta)^2\right]} .
\end{align*}
Furthermore, for two congruent sets $\cB_1 \subset \cB_2$, the variance of $G_v$ using $B^*(\cB_1)$ is 
larger or equal than when using $B^*(\cB_2)$.
\end{theorem}

This form of the optimal baseline is folklore, and can be found (in the context of RL) in \citep{greensmith2004variance} and \citep{ranganath2013black} (in the context of VI).

\subsection{Choosing conditioning sets of value functions and critics}

From our main theorem~\ref{thm:value-critic}, we know that valid critic and baseline sets together define a new estimate for the gradient of a stochastic computation graph. But how should one choose the critic and baseline sets? We discuss here the tradeoffs made when choosing different sets.

\begin{addmargin}[1em]{2em}
$\bullet\:\:$ Given a critic with conditioning set $\cC$, how to choose an appropriate set $\cB$ for the value?\end{addmargin}

For congruent optimal baselines, theorem \ref{thm:optimal-baseline} states that variance is always reduced by increasing the conditioning set $\cB$. For the optimal baseline, this implies it is optimal to choose $\cB$ to be the intersection of $\cC$ and the set of non-descendants of $v$. 
A related statement happens to be true for value baselines:
\begin{lemma}\label{lem:baseline-variance}
Consider two congruent baseline sets $\cB_1\subset \cB_2\subset \cC$. Then the variance of the advantage using $\cB_1$ is higher than that using $\cB_2$:
\begin{align*}
    \E\left[\left(Q(\cC)-V(\cB_2)\right)^2\right]\leq \E\left[\left(Q(\cC)-V(\cB_1)\right)^2\right]
\end{align*}
\end{lemma}

\begin{addmargin}[1em]{2em}
$\bullet\:\:$ Is there any point in using non-congruent value functions (i.e. using a conditioning set $\cB$ which is not a subset of the critic set $\cC$)? \end{addmargin}

The previous lemma suggests, but does not prove, that value functions should also be conditioned on the maximal congruent set $\cB^*$, intersection of $\cC$ with the non-descendants of $v$. One may wonder the usefulness of non-congruent value functions. Intuitively, variables in $\cB \setminus \cC$ have already been averaged in the critic, so it would be tempting to disregard non-congruent baselines.
The answer is here more complicated; non-congruent baselines may increasing or decrease variance, depending on correlations between value and critics, as we see in the following two examples.

\begin{figure}[h!]
    \centering
    \includegraphics[width=0.45\textwidth]{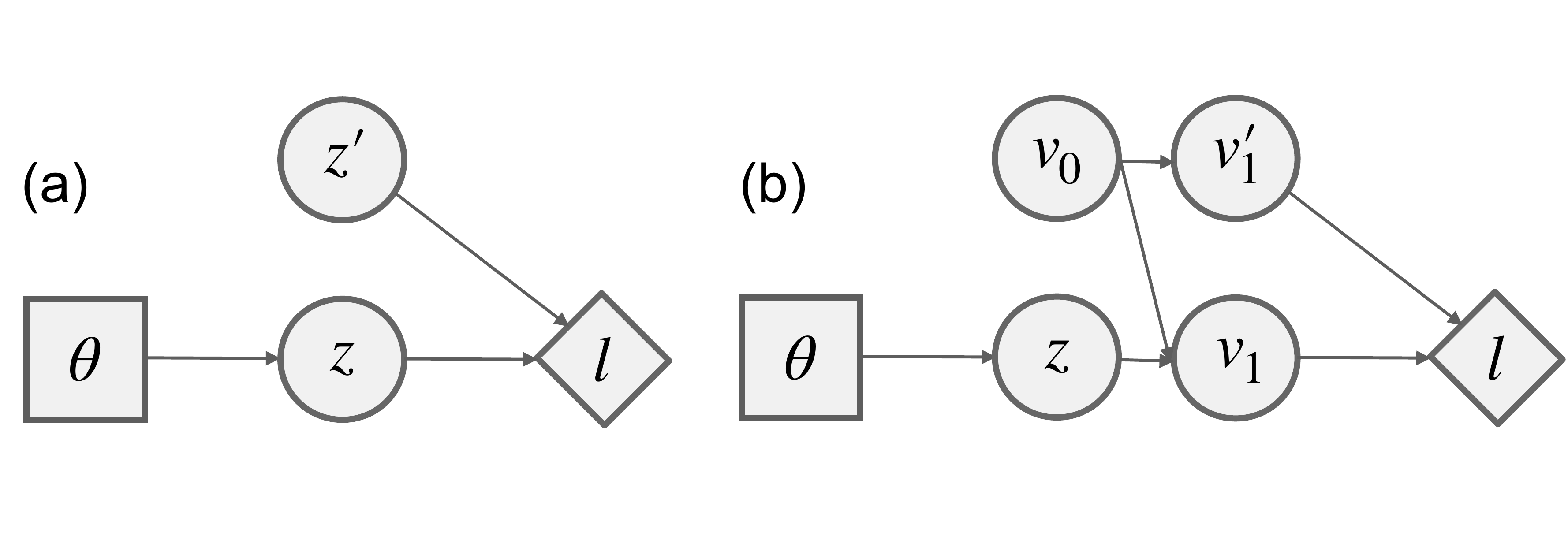}
    \caption{\emph{Constructing value sets}.
    In (a), non-congruent baselines increase variance, in (b), they decrease it.
    \label{fig:conditioning-sets}}
\end{figure}

\begin{example}
Consider the example of figure~\ref{fig:conditioning-sets}(a), where we consider two variables $z\sim p_\theta(z), z'\sim \mathcal{N}(0,\sigma)$, a loss function $\ell(z,z')=\ell(z)+z'$. Assume that the variance $\sigma$ of $z'$ is large, that all variables are observed. Because $z'$ has large variance, it tends to dominate the loss, even if it is not controllable by changing $\theta$. The gradient estimate for $\DbyDt \E_{z,z'} \left[ \ell(z,z')\right]$ is given, per theorem \ref{thm:value-critic}, by $\DbyDt \log p_\theta(z) (Q(\cC)-V_z(\cB)$.

The `naive' estimate consists in using the empirical loss $(\ell(z)+z')$ as critic (which is equivalent to choosing $\cC=\{z,z'\}$) and a  baseline equal to $\E[\ell(z)]$ (obtained by choosing the value set $\cB=\varnothing$); the resulting advantage $\ell(z)-\E[\ell(z)]+z'$ has high variance due to $z'$. A lower variance estimate can be obtained by choosing either $\cC=\{z\}$ and maximally congruent $\cB=\varnothing$, or $\cC=\{z,z'\}$ and maximally congruent $\cB=\{z'\}$; both lead to the same low-variance estimate $\ell(z)-\E[\ell(z)]$ of the advantage.

However, if we use $\cC=\{z\}$ and non-congruent $\cB=\{z'\}$, we obtain the high variance estimate again.
\end{example}

\begin{example}
Consider next the example in figure~\ref{fig:conditioning-sets}(b), with $z\sim p_\theta(z)$, $v_0\sim \mathcal{N}(0,\sigma)$;  $v_1$ and $v_1'$ are conditionally independent with respective distributions $\mathcal{N}(\ell(z)+v_0, \sigma')$ and $\mathcal{N}(v_0, \sigma')$, with $\sigma'<<\sigma$.
Suppose that $\ell(v_1,v_1')=v_1+v_1'$ and that all variables but $v_0$ are observed.
Taking for instance critic set $\cC=\{z,v_1\}$ and congruent set $\cB=\varnothing$, we have $Q(\cC)-V(\cB)=(\ell(z)-\E[\ell(z)])+2(v_1-\ell(z))$, which is high variance because $(v_1-\ell(z))$ has variance at least as high as $v_0$. 

Note however that if we choose the non-congruent set $\cB=\{v_1'\}$, then $V(\cB)=\E[\ell(z)]+2v_1'$, so that $Q-V=(\ell(z)-\E[\ell(z)])+2(v_1-\ell(z)-v_1')$,
which is much lower variance since $(v_1-\ell(z)-v_1')$ only has variance $\sigma'$. 
\end{example}

\begin{addmargin}[1em]{2em}
$\bullet\:\:$ Having looking at the conditioning set of the value function, we look at the following question: how should we construct a critic conditioning set $\cC$? \end{addmargin}

A Rao-Blackwellization argument suggests that marginalizing out as many variables as possible leads to the lowest variance. However, this is assuming exact conditional expectations. However, in practice those are not easily computed and we will resort to function approximation (see section~\ref{sec:comp-val-crit}). When using function approximations to our estimates of value and functions, there will be a bias-variance trade-off between using approximate critic which marginalizes out random variable (which leads to bias) or using model samples of the losses (which leads to variance). Furthermore, the re-usability of a particular value function for multiple purposes (e.g.\ to serve as part of critics for multiple nodes, and / or as bootstrap target(s), see below) can motivate more limited marginalization).

\subsection{Critic correction when using gradient-critics}
\label{sec:debiasing}

A common pattern is to use critic to approximate a model that has no gradient, or whose gradients are unknown. In the case we use a Markovian set for the critic, we can learn the critic from predicted losses only (i.e. set $\alpha$ to $0$ in equation~\ref{eq:Sobolev}), but then use the corresponding gradient as a gradient-critic in the reparameterized graph. The use of the critic or gradient critic incurs a bias in the estimator. Here we see a technique which combines gradient-critics and critics in order to build an unbiased estimator which still leverages the gradient-critic (this technique can in fact be used anytime a gradient-critic is used)
Consider a stochastic node $v$ which can be reparametrized, let $L(v)$ be the sum of all downstream losses from $v$. Consider a Markovian critic set $\cC_v$ for $v$ and let $Q(\cC_v)$ be the corresponding value critic, and $\hat Q(\cC_v)$ an approximation of it. Let $\theta_v$ be an arbitrary parent of $v$.
From theorem~\ref{thm:gradientCritic:criticGradient}, $\DxbyDy{Q(\cC_v)}{v}$ is a valid gradient-critic for $v$, and we have $$\DxbyDy{\E[L(v)]}{\theta_v}=\E\left[\DxbyDy{\log p(v)}{\theta_v}L(v)\right]=\E\left[\DxbyDy{\log p(v)}{\theta_v}Q(\cC_v)\right]=\E\left[\DxbyDy{Q(\cC_v)}{v}\DxbyDy{v}{\theta_v}\right]\approx \E\left[\DxbyDy{\log p(v)}{\theta_v}\hat Q(\cC_v)\right]$$ 
The second equality is the regular score function estimator; the second, the score function estimator using critics; the third, a reparametrized estimator using gradient-critics. The last equality is the same as the second, but taking into account the bias induced by the use of an approximate critic.

By writing the loss $L(v)$ as the sum of two terms $(L(v)-Q(\cC_v))+Q(\cC_v)$, we can use the score function estimator on the first term,  reparameterize the second term, and therefore potentially leverage the lower variance of the reparameterized gradient while keeping unbiasedness thanks to the score function term. The resulting estimator will take the following form:
$$\DxbyDy{\E[L(v)]}{\theta_v}=\E\left[\DxbyDy{\log p(v)}{\theta_v}\left(L(v)-\hat Q(\cC_v)\right)+\DxbyDy{\hat Q(\cC_v)}{v}\DxbyDy{v}{\theta_v}\right].$$ 
This estimator is strongly related to Stein variational estimator, see e.g.\ \citep{liu2016stein}.
Note if the critic is exact, i.e.\ $Q=\hat Q$, the first term has zero expectation and may be excluded.
A common use pattern is when $v$ is an action $a$ (in RL) or sample $z$ (in generative models), the critic set is $(s,a)$, where $s$ is the state (in RL) or context/previous state (in generative models).
This technique sometimes called `action-conditional baselines' in the literature, see for instance \citep{tucker2017rebar,gu2015muprop,grathwohl2017backpropagation}. Is it however not clear that the bias correction does not increase the variance in such a way that the gradients from reparameterization become dominated by the noise of the score function \citep{tucker2018mirage}.
Interestingly, the action-conditional baseline literature takes an almost opposed interpretation to ours: while we see the $(L(v)-\hat Q(\cC))$ as a correction to the biased estimator resulting using an approximate critic $\hat Q(\cC)$, \citep{tucker2017rebar,grathwohl2017backpropagation} view instead the critic as an `invalid' baseline (as in, it biases the gradient of the score function estimator), which needs to be corrected with the reparametrized term $\DxbyDy{\hat Q(\cC_v)}{v}\DxbyDy{v}{\theta_v}$. In our view, this interpretation suffers from not having a natural way for trading off bias and variance, for instance weighting the correction term $\left(L(v)-\hat Q(\cC_v)\right))$ with a coefficient less than $1$ or leaving it out entirely. 

\section{Proofs}

For any node $v$, let $\bcG_v$ be the set of non-descendants of $v$.

\subsection{Computation lemmas}

We provide a couple of useful lemmas describing properties of the computation resulting from a stochastic computation graph. 

Our first lemma describes a set which deterministically computes $\ell$ given information in a set $\cC$, where $\cC$ is `as far away' from $\ell$ as possible.

\begin{lemma}\label{lem:Decomp}
Consider a cost $\ell$ and arbitrary set $\cC$ for which $\ell$ descends from $\cC$. Let a sequence of set $\cV_i$ be described as follows:$\cV_0=\{\ell\}$, and for any $i\geq 1$, we let $\cV_i$ be defined as follows:
\begin{itemize}
    \item Any stochastic or input node in $\cV_{i-1}$ is included in $\cV_i$.
    \item Any node in $\cC \cap \cV_{i-1}$ is included in $\cV_i$.
    \item The parents of any deterministic node in $\cV_{i-1} \setminus \cC$ are included in $\cV_i$.
\end{itemize}
Then, for any $i$, $\ell$ can be deterministically computed from $\cV_i$. Furthermore, the set converges to a set $\cV$ which only contains stochastic nodes, input nodes, and nodes in $\cC$. 
\end{lemma}

\begin{proof}
The first part of the proof is a simple recursion. $\ell$ is clearly a deterministic function of $\ell$. Assume now that $\ell$ is a deterministic function of $\cV_{i}$, and consider the set $\cV_{i+1}$. For any variable $v$ in $\cV_{i}$, it is either in $\cV_{i+1}$, or a deterministic function of its parents, which are included in $\cV_{i+1}$. It follows that $\ell$ is a deterministic function of $\cV_{i+1}$. For proving convergence, the set of nodes which are are in $\cV_i$ and either in $\cC$, stochastic, or input nodes is non-decreasing (since nodes of that type stay in $\cV_i$ if they are in $\cV_{i-1})$. As for deterministic nodes, they can only be in $\cV_i$ if they are in $\cC$ or if they are in a directed path of distance $i$ to $\ell$. This means that for $j$ greater than the the longest directed path between a node and $\ell$, $\cV_j$ can only contain deterministic nodes if they are in $\cC$ (which proves the second part of the lemma). Together, these statements imply that for $j'\geq j$, the set $\cV_{j'}$ is non-increasing, and it is bounded, therefore it converges. 
\end{proof}

We provide a pictorial representation of the algorithm for a given graph as follows:

\begin{figure}[H]
    \centering
    \includegraphics[width=0.9\textwidth]{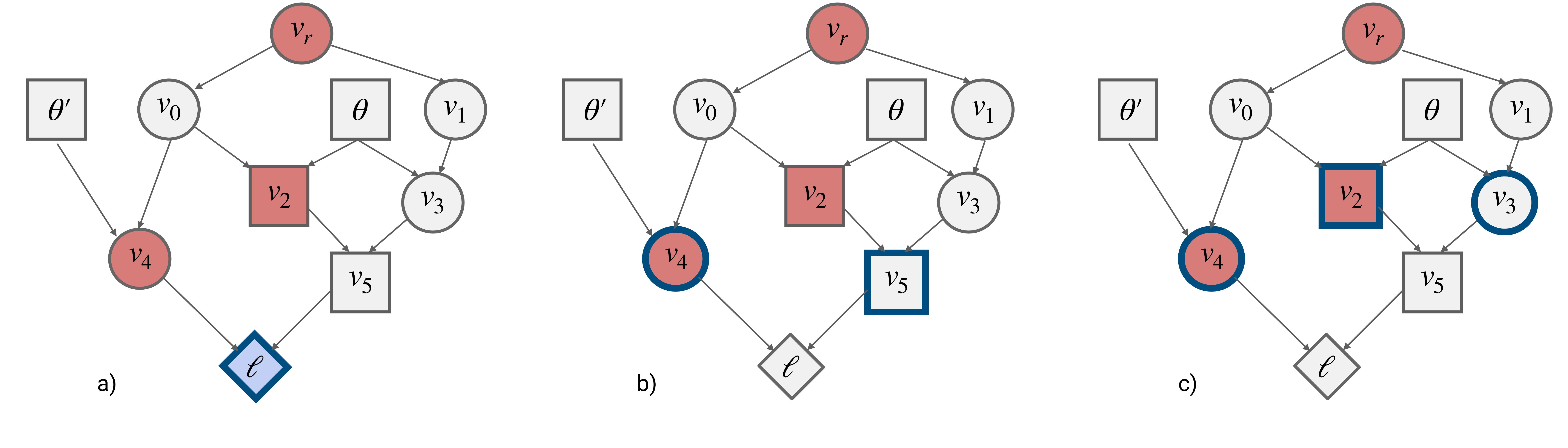}
    \caption{Illustration of the algorithm described in Lemma~\ref{lem:Decomp}. The set $\cC$ is indicated by nodes with dark red background, the sets $\cV_i$ by nodes with  blue frame. The algorithm finishes at the third step, since the resulting set is stable. $\ell$ can indeed be deterministically computed from $v_2, v_3$ and $v_4$.}
    \label{fig:root_algo}
\end{figure}

When $\cC$ is Markovian, the set $\cV$ has an important property which we will later use.

\begin{property}\label{prop:MarkovDecomp}
Let $\cW(\cV)$ be the set of stochastic ancestors of nodes in $\cV$, unblocked by $\cC$. We have:
\begin{eqnarray}
P(\cW|\cC)=\prod_{w\in cW}p(w|h_w(\cC))
\end{eqnarray}
where the notation $h_w(\cC)$ is simply to indicate that the simple forward computation of $h_w$ depends on the values of variables in $\cC$ which are ancestors of variables in $\cW$.
\end{property}

In other words, to sample all the variables in $\cW$, we can simply use the forward model; the conditioning set $\cC$ only plays a role in setting the value of ancestors for some of the variables in that set. 

In the example above, the set $\cW$ is $\{v_3, v_1\}$; this is because $v_3$ is in $\cV$ (but not in $\cC$) and $v_1$ is an ancestor of $v_3$, unblocked by $\cC$. The node $v_r$ is \emph{not} included in $\cW$ because it is in $\cC$.
Conditional on $\cC$, the distribution of $\cW$ is given by $p(v_1|v_r)p(v_3|v_1)$.

\begin{proof}[Proof of Prop.~\ref{prop:MarkovDecomp}]
By the definition of Markovianity, nodes in $\cW$ cannot have a descendent in $\cC$. The property follows from applying Bayes' rule.
Using a graph-compatible ordering of variables in $w_i$, we have: $P(\cW|\cC)=\prod_i p(w_i|\cC, w_{<i})=\prod_i \frac{p(w_i, \cC, w_{<i})}{p(\cC, w_{<i})}$. 
At this point we simply write numerator and denominators as product of probabilities of variables conditioned on parents; because none of the variables in $\cC$ descend from a variable in $\cW$, all terms cancel out except for $p(w_i|h_{w_i})$, which gives us the desired result.
\end{proof}
Our second lemma relates conditional independence between variables to conditional independence between derivatives.

\begin{lemma}\label{lem:cond-indep-grad}
Assume that $\log p_v$ is conditionally independent of cost $\ell$ given $\cC$. Then $\DxbyDy{\log p_v}{\theta}$ is conditionally independent of $\ell$ given $\cC$.
\end{lemma}
\begin{proof}
By forward mode differentiation, the stochastic computation which describes the computation of $\cC$, $\ell$ and $\log \pi_v$ can also describe the computation of $\cC, \ell$ and $(\log p_v, \DxbyDy{\log p_v}{\theta})$. If the graph properties imply conditional independence of $\log p_v$ and $\ell$ given $\cC$, they therefore also imply conditional independence of $\ell$ and $\DxbyDy{\log p_v}{\theta}$ given $\cC$. Note this is only true because we only consider graph-induced conditional independence properties.
\end{proof}

For similar reasons, we have:
\begin{lemma}\label{lem:cond-ind-grads}
Let $u,v,w$ be nodes in a graph, and let $\cC$ be an arbitrary set. $\DxbyDy{v}{u}$ and $\DxbyDy{w}{v}$ are conditionally independent given $\cC$ if $\{u': u' \text{ is in a path between $u$ and $v$}\}$ and $\{v': v' \text{ is in a path between $v$ and $w$}\}$ are conditionally independent given $\cC$.
\end{lemma}

\subsection{Proofs of results for value-based methods}\label{sec:appendix-value}

\begin{proof}[Proof of Theorem~\ref{thm:value-critic}]

We prove theorem~\ref{thm:value-critic} by proving that for any node $v$ in $\cS$, $\E\left[\DxbyDy{}{\theta}\log p_v L\right]=\E[q_v]$. To do so, we follow three steps:
1) show that $L$ can be replaced by $L(v)$, 2) show that any baseline $B$ can be subtracted from $L(v)$ (and hence, a value baseline can be subtracted), and 3) show that $L(v)$ can be replaced by a critic $Q(\cC)$.

For any node $v$ in $\cS$, consider any cost $\ell$ non-descendant from $v$, i $\ell \in \bcG_v$. the contribution of $v$ to the gradient estimate with respect to $\ell$ is $\E\left[\DbyDt \log p_v \ell \right]
$. Using $\cG=\bcG_v\cup (\cG\setminus \bcG_v)$ and the law of iterated expectations, we obtain:
\begin{align*}
\E\left[\DbyDt \log p_v \ell \right] =& \E_{\bcG_v}\left[\E_{\cG\setminus \bcG_v| \bcG_v}\left[\DbyDt \log p_v \ell\right]\right] \\
=&\E_{\bcG_v}\left[\ell\:\E_{\cG\setminus \bcG_v | \bcG_v} \left[ \DbyDt \log p_v \right]\right] = 0
\end{align*}
The second equality comes from the fact that this particular $\ell$ is a deterministic function of $\bcG_v$ (since it is a non-descendant of $v$). The third equality comes from the following:
\begin{align*}
    E_{\cG\setminus \bcG_v | \bcG_v} \left[ \DbyDt \log p_v \right]= 
   & \int_v p(v|h_v) \DbyDt \log p(v|h_v) \\=
   & \int_v p(v|h_v) \frac{\DbyDt p(v|h_v)}{p(v|h_v)} \\  
  = & \int_v {\DbyDt p(v|h_v)} = \DbyDt \int_v p(v|h_v) =\DbyDt 1 = 0
\end{align*}
It is important to note that conditional on the non-descendants $\bcG_v$ (which includes the parents $h_v$ of $v$, but no descendants) the distribution of $v$ is simply $p(v|h_v)$. 
Using the same idea, we show that subtracting a baseline does not bias the expectation:
\begin{align*}
 \E\left[\DbyDt \log p_v V(\cB) \right] = & \E_{\bcG_v}\left[\E_{\cG\setminus \bcG_v| \bcG_v}\left[\DbyDt\log p_v B(\cB)\right]\right] =  \E_{\bcG_v}\left[B(\cB) \E_{\cG\setminus \bcG_v| \bcG_v}\left[\DbyDt\log p_v \right]\right] = 0
\end{align*}
Finally,
\begin{align*}
\E\left[\DbyDt \log p_v L(v) \right] = & \E_{\cC_{v}}\left[\E_{\cG \setminus \cC_{v}|\cC_{v}}\left[\left(\DbyDt \log p_v\right) L(v)\right]\right] \\
= & \E_{\cC_{v}}\left[\E\left[\DbyDt\log p_v|\cC\right] \E\left[ L(v)|\cC\right]\right]\\
= & \E_{\cC_{v}}\left[\E\left[\DbyDt\log p_v|\cC\right] Q(\cC)\right]\\
= & \E_{\cC_{v}}\left[\E\left[Q(\cC)\DbyDt\log p_v|\cC\right] \right]\\
= & \E\left[Q(\cC)\DbyDt\log p_v\right]
\end{align*}
where the second line comes from the conditional independence of $L(v)$ and $\DbyDt\log p_v$ given $\cC$ (following the conditional independence of $L(v)$ and $\log p_v$, and applying lemma~\ref{lem:cond-indep-grad}), the third line follows from the definition of the critic, and the fourth from the fact that the critic is constant as a function of $\cC$.

\end{proof}

\begin{proof}[Proof of Theorem~\ref{thm:optimal-baseline}]

The variance of $q_v$ is $\Var(q_v) = \E[q_v^2]-\E[q_v]^2$. Since the later term is unaffected by the choice of the baseline function, minimizing the variance is equivalent to minimizing $\E[q_v^2]$. 
\begin{align*}
    \E[q_v^2] =& \E\left[s_\theta(v,h_v)^2 \left(Q(v,\cC)-V(\cB)\right)^2\right]\\
    =& \sum_{b} p(\cB=b) \left(\E_{\cG \setminus \cB | \cB=b}\left[s_\theta^2 \left(Q(v,\cC)-B(b)\right)^2\right]\right)
\end{align*}
The expression above is a sum of non-negative expressions, each involving a distinct $B(b)$ to optimize. The sum is jointly minimized if each is minimized.
For a given $b$, we take the gradient of $\E_{\cG \setminus \cB | \cB=b}\left[s_\theta^2 \left(Q(v,\cC)-B(b)\right)^2\right]$ with respect to $B(b)$ and set it to zero; we find:
\begin{align*}
    B^*(\cB) = \frac{\E_{\cG\setminus \cB|\cB}\left[s_\theta^2 Q(v,\cC)\right]}{\E_{\cG\setminus \cB|\cB}\left[s_\theta^2\right]}
\end{align*}
To prove the second equality, simply note that $\cB \subset \cC$ (since the baseline is congruent); it follows that:
\begin{align*}
\E_{\cG\setminus \cB|\cB}\left[s_\theta^2 L(v) \right]  = &\: \E_{\cC\setminus \cB|\cB} \left[\E_{\cG\setminus \cC | \cC} s_\theta^2 L(v)\right]\\
= &\: \E_{\cC\setminus \cB|\cB} \left[s_\theta^2 Q(v,\cC) \right]
\end{align*}

Next, we consider two conditioning sets $\cB_1 \subset \cB_2$, with respective optimal baselines $B^{*}_{1}$ and $B^{*}_{2}$ (as defined above).
For an assignment $b$ of variables in $\cB_2$, let ${b}_{|1}$ be the restriction of $b$ to the variables found in $\cB_1$. We can construct a baseline $B$ with set $\cB_2$ by choosing $V(b)=V^*_1({b}_{|1})$; this is a valid baseline for set $\cB_2$, but is strictly equivalent to using the optimal baseline $B^*_1$ for set $\cB_1$. By optimality of $B^*(\cB_2)$, the variance of $q_v$ using $B^*(\cB_1)$ is higher than using $B^*(\cB_2)$.
\end{proof}

\begin{proof}[Proof of lemma~\ref{lem:average}]
Consider $\cX_v^1\subset \cX_v^2$. By definition of the conditional expectation, we have
$V(\cX_v^1)=\E[L(v)|\cX_v^1]$. By the law of iterated expectation, $\E[L(v) | \cX_v^1]=\E[\E[L(v)|\cX_v^2]|\cX_v^1]=\E[V(\cX_v^2)|\cX_v^1]$. 
\end{proof}

\begin{proof}[Proof of property~\ref{prop:Markov}]
Applying property~\ref{prop:MarkovDecomp}, we can see that the conditional distribution of $\cW$ given $\cX_v$ and $\cX_v^\uparrow$ is the same (direct stochastic ancestors to a variable in $\cW$ cannot be in $\cX_v^\uparrow\setminus \cX_v)$, therefore the conditional expectations are the same. By law of iterated expectations, for $\cX_v\subset \cX_v' \subset \cX_v^\uparrow$, $V(\cX_v')=\E[\E[L(v)|\cX_v^\uparrow]|\cX_v']=\E[V(\cX_v^\uparrow)|\cX_v']=\E[V(\cX_v)|\cX_v']=V(\cX_v)$.
\end{proof}

\begin{proof}[Proof of lemma~\ref{lem:Bellman}]
Since $\cX^1\subset \cX^{2\uparrow}$, by Lemma~\ref{lem:average}, we have $V(\cX^1)=E[V(\cX^{2\uparrow})|\cX^1]$. But since $\cX^2$ is Markov, $V(\cX^{2\uparrow})=V(\cX^2)$ and the lemma follows.
\end{proof}

\begin{proof}[Proof of theorem~\ref{thm:bootstrap}]
By the decomposition assumption, we have:
\begin{align*}
    L(v) = \sum_i L(V_i)
\end{align*}
Taking expectations conditional on $\cX$, we obtain
\begin{align*}
    V(\cX)=\E[L(v) | \cX] = \sum_i\E\left[L_{v_i}|\cX\right]
\end{align*}

For each $i$, $V(\cX_{V_i})=\E[L(V_i)|\cX_{V_i}]$; by lemma~\ref{lem:Bellman}, $\E\left[L({V_i})|\cX\right]=\E[V(\cX_{V_i}|\cX]$.
\end{proof}

\begin{proof}[Proof of Lemma~\ref{lem:baseline-variance}]
This follows from classical manipulation regarding conditional expectation:
\begin{align*}
    \E\left[(Q(\cC)-V(\cB_2))^2|\cB_2\right]= &
    \E\left[\Big(\big(Q(\cC)-V(\cB_1)\big)+\big(V(\cB_1)-V(\cB_2)\big)\Big)^2\Big|\cB_2\right]\\
    = & \E\Big[\big(Q(\cC)-V(\cB_1)\big)^2+\big(V(\cB_1)-V(\cB_2)\big)^2\\
    & +2\big(Q(\cC)-V(\cB_1)\big)\big(V(\cB_1)-V(\cB_2)\big)\Big|\cB_2\Big]
\end{align*}
The last term can be simplified:
\begin{align*}
    \E\left[\big(Q(\cC)-V(\cB_1)\big)\big(V(\cB_1)-V(\cB_2)\big)\Big|\cB_2\right]& =\big(V(\cB_1)-V(\cB_2)\big)\E\left[ Q(\cC)-V(\cB_1)\Big|\cB_2\right]\\
    &=-\big(V(\cB_1)-V(\cB_2)\big)^2
\end{align*}
The first equality is due to the fact that since $\cB_1\subset \cB_2$, conditioned on $\cB_2$, both values are constant; the second equality due to $\E[Q(\cC)|\cB_2]=V(\cB_2)$ due to $\cB_2\subset \cC$ and the law of iterated expectations. We obtain $\E\left[(Q(\cC)-V(\cB_2))^2|\cB_2\right] = \E\left[\big(Q(\cC)-V(\cB_1)\big)^2-\big(V(\cB_1)-V(\cB_2)\big)^2\Big|\cB_2\right]$, which from applying the law of iterated expectation again becomes
\begin{align*}
    \E\left[(Q(\cC)-V(\cB_2))^2\right] = \E\left[\big(Q(\cC)-V(\cB_1)\big)^2-\big(V(\cB_1)-V(\cB_2)\big)^2\right] \leq \E\left[\big(Q(\cC)-V(\cB_1)\big)^2\right]
\end{align*}
\end{proof}

\subsection{Proofs of results for gradient-based methods}\label{sec:appendix-gradient}

\begin{proof}[Proof of theorem~\ref{thm:grad-critic}]
Using the iterated law of expectations once again:
\begin{align*}
    \E_{\cG}\left[\sum_i \DxbyDy{L^s}{v_i} \dxbydy{v_i}{v}\right] & = \sum_i \E_{\cC_{v_i}}\left[\E_{\cG | \cC_{v_i}}\left[ \DxbyDy{L^s}{v_i} \dxbydy{v_i}{v}\right]\right]\\
    &= \sum_i \E_{\cC_{v_i}}\left[\E_{\cG | \cC_{v_i}}\left[ \DxbyDy{L^s}{v_i} \right]\:\E_{\cG | \cC_{v_i}}\left[\dxbydy{v_i}{v}\right]\right]
    =\sum_i \E_{\cC_{v_i}}\left[ g_{v_i}\: \E_{\cG | \cC_{v_i}}\left[\dxbydy{v_i}{v}\right] \right]\\
    &=\sum_i \E_{\cC_{v_i}}\left[ \E_{\cG | \cC_{v_i}}\left[g_{v_i}\dxbydy{v_i}{v}\right] \right]=\sum_i \E\left[g_{v_i}\dxbydy{v_i}{v}\right]
\end{align*}
The third equality comes the conditional independence assumption, the fourth from the definition of the gradient-critic, and the fourth from the fact that the gradient-critic is a deterministic function of $\cC_{v_i}$.
\end{proof}

\begin{proof}[Proof of theorem~\ref{thm:gradientCritic:criticGradient}]

From the Markov assumption, following Lemma~\ref{lem:Decomp} and property~\ref{prop:MarkovDecomp}, we can write $L(v)$ as a deterministic function of a set $\cV \ni v$ with stochastic ancestors $\cW$ with $P(\cW|\cC)=\prod_{w \in \cW}p(w|h_w(\cC))$.

We have $Q(\cC)=\int_{\mathbf{w} \in \cW} \prod_{w \in \cW}p(w|h_w)  L(v)(\cV)$.
Taking the gradient with respect to $v$, we obtain:
\begin{align*}
    \DxbyDy{Q(\cC)}{v} & = \int_{\mathbf{w} \in \cW} \left(\prod_{w \in \cW} p(w|h_w)\right) \left(\DxbyDy{L(v)}{v}+\sum_{w \in \cW} \DxbyDy{\log p(w|h_w)}{v} L(v)\right)\\
    & = \int_{\mathbf{w} \in \cW} \left(\prod_{w \in \cW} p(w|h_w)\right) \DxbyDy{}{v}\left(L(v)+\sum_{w \in \mathcal{A}^S} \log p(w|h_w)L(v)\right)\\
    & = \int_{\mathbf{w} \in \cW} \left(\prod_{w \in \cW} p(w|h_w)\right) \DxbyDy{}{v}L^s_v=\E\left[\DxbyDy{}{v}L^s_v|\cC\right]=g_v(\cC)
\end{align*}
\end{proof}

\end{document}